\documentclass[runningheads]{llncs}

\usepackage{eccv}

\usepackage{eccvabbrv}

\usepackage{graphicx}
\usepackage{booktabs}
\usepackage{multirow}
\usepackage{pifont}

\usepackage[accsupp]{axessibility}  % Improves PDF readability for those with disabilities.

\usepackage{hyperref}

\usepackage{orcidlink}

\begin{document}

% ---------------------------------------------------------------
% REVIEW: Replace with your title
\title{PUF: Plug-and-Play Uncertainty-Aware Fusion for Online 3D Scene Graph Generation} 

% REVIEW: If the paper title is too long for the running head, you can set
% an abbreviated paper title here. If not, comment out.
\titlerunning{PUF for Online 3D Scene Graph Generation}

% FINAL: Replace with your author list. 
% Include the authors' OCRID for the camera-ready version, if at all possible.
\author{Yi Yang\inst{1}\orcidlink{0009-0001-1362-2269} \and
Myrna Castillo\inst{2,3}\orcidlink{0000-0001-6814-261X} \and
Bodo Rosenhahn\inst{1}\orcidlink{0000-0003-3861-1424} \and \\
Michael Ying Yang\inst{3}\orcidlink{0000-0002-0649-9987}\thanks{Corresponding author.}}

% FINAL: Replace with an abbreviated list of authors.
\authorrunning{Y. Yang et al.}
% First names are abbreviated in the running head.
% If there are more than two authors, 'et al.' is used.

\institute{
Leibniz Universität Hannover, Germany \and
Istituto Italiano di Tecnologia, Italy \and
University of Bath, UK
}

\maketitle

\begin{abstract}
Online 3D scene graph generation builds a persistent, structured representation of a scene by incrementally fusing 2D observations into a global 3D graph. 
Existing online methods treat this fusion as a fully deterministic pipeline,
where we identify three sources of uncertainty that are overlooked: observation, 2D model, and 3D representation.
We propose PUF: a Plug-and-play, Uncertainty-aware, and training-free Fusion framework. 
Scene graph node association is reformulated as a probabilistic likelihood over semantic and spatial factors, replacing binary accept/reject gates.
Dirichlet evidence accumulation distributes class and relationship evidence across plausible candidates proportional to association likelihood. An optional class-conditional prior completes edges for sparsely or never co-observed object pairs. 
We instantiate PUF with both a 3D Gaussian and a 3D voxel backend and observe consistent improvements,
demonstrating its ability to generalize across different representations.
Experiments on the 3DSSG and ReplicaSSG benchmarks show that our method substantially outperforms existing approaches while maintaining real-time latency. 
These results establish uncertainty-aware fusion as a principled and effective paradigm for online 3D scene understanding.
%Codes will be released upon publication.
The source code is publicly available at \url{https://github.com/yyyyangyi/PUF}.
\end{abstract}

% ============================================================
\section{Introduction}
\label{sec:intro}
% ============================================================

A 2D Scene Graph (SG) describes an image as a structured representation in which nodes correspond to detected objects and directed edges encode their pairwise relationships. A 3D scene graph lifts this abstraction into metric space. 
The task of online 2D-to-3D Scene Graph Generation (SGG) is to construct a global 3D SG incrementally from a stream of RGB-D frames. 3D SGs provide the high-level scene abstraction needed for downstream tasks such as embodied navigation~\cite{singh2023scene}, robotic manipulation~\cite{zhu2021hierarchical}, and spatial question answering~\cite{wu20243d}. Incremental construction supports real-time applications for which a complete scene reconstruction is unavailable or impractical~\cite{shah2021ving}, and lifting 2D predictions online into 3D leverages powerful 2D SGG models trained on richly annotated datasets~\cite{krishna2017visual} without the prohibitive cost of 3D SG annotation~\cite{wald2019rio,wald2020learning,yang20234d}.

In the online setting, a global 3D SG is built by fusing a continuous stream of 2D observations, \eg, a video sequence, into a persistent graph structure. With the camera's field of view covering only a fraction of the scene at any instant, every such observation is partial and inherently uncertain. Accurately fusing these partial observations is non-trivial, as it requires reasoning about \emph{how much} to trust each observation. However, existing online methods discard this information by committing to hard decisions at every stage. We identify three distinct sources of uncertainty, as illustrated in Fig.~\ref{fig:teaser}. 
\emph{(a)} Observations are uncertain as objects are truncated. Relationships are particularly subject to such noise since they can only be predicted when both endpoint objects are simultaneously well observed.
\emph{(b)} 2D SGG models are uncertain. They produce soft class and relationship distributions that cannot be collapsed to hard predictions without loss. 
\emph{(c)} Object representation in 3D is approximate. Back-projecting a 2D bounding box through a single noisy depth frame without committing to a full 3D reconstruction yields uncertain 3D position and spatial extent. 

\begin{figure}[htbp!]
    \centering
    \includegraphics[width=0.99\linewidth]{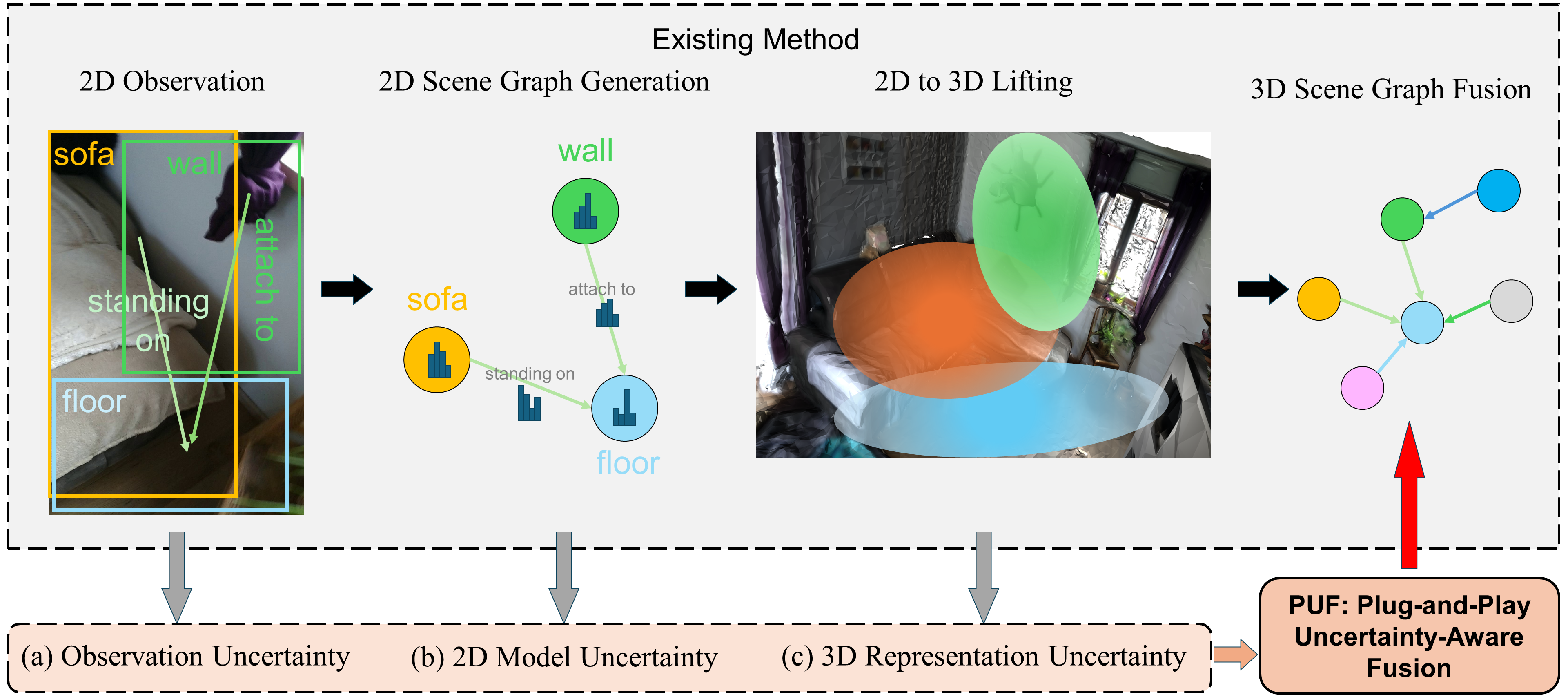}
    \caption{
        Our method is aware of 3 types of uncertainty which are overlooked in existing online 3D SGG methods. (a) Uncertainty from partial observation in 2D; (b) Uncertainty encoded in softmax output from 2D SGG model; (c) Uncertainty in 3D lifting, e.g. 3D Gaussian with estimated depth in FROSS~\cite{hou2025fross}. 
    }
    \label{fig:teaser}
    \vspace{-3mm} %
\end{figure}

Offline 3D SGG methods rely on 3D point cloud input and global message-passing over the reconstructed graph, making them accurate but incompatible with real-time operation~\cite{wu2021scenegraphfusion,feng20233d,koch2024sgrec3d,heo2025objectcentric}. Online methods~\cite{kim20193,wu2023incremental,gu2024conceptgraphs,feng2025hyperrectangle,hou2025fross,yeo2025statistical} process frames incrementally and achieve low latency, but treat the 2D-to-3D merging stage as a fully deterministic pipeline. This rigidity overlooks the sources of uncertainty above and therefore introduces three systematic limitations. 
\emph{(a)} A detection is either merged or rejected based on its argmax label, which prevents 2D model uncertainty from propagating into 3D.
\emph{(b)} Thresholding on raw 3D distance or spatial overlap treats back-projected 3D positions as exact, ignoring the inherent ambiguity in depth-based localization. 
\emph{(c)} Hard label transfer during merging does not redistribute evidence across candidate nodes, nor complete relational evidence for sparsely co-observed edges.

We propose \textbf{PUF}: a \emph{plug-and-play}, \emph{uncertainty-aware} and  \emph{training-free}  fusion framework that compensates for these failures simultaneously. Our framework wraps any 2D SGG model that outputs soft class and relation probability distributions, and propagates these distributions directly into our proposed merging process. 2D model uncertainty is hence preserved throughout 3D reconstruction. 
Scene graph node association is formulated as a probabilistic problem that jointly accounts for semantic uncertainty and 3D representation uncertainty. Each observation is matched against existing nodes through a likelihood, combining class-distribution similarity with a spatial overlap term, so that uncertain geometry attenuates rather than vetoes an association. 
Upon association, both node semantics and edge labels are updated through Dirichlet accumulation that redistributes evidence across all plausible candidates proportional to association likelihood. An optional relation prior compensates for poorly observed or structurally unobservable edges. 
PUF is representation-agnostic, and we instantiate it with both a 3D Gaussian and a 3D voxel backend to demonstrate 
its ability to generalize across different representations.
We evaluate PUF on the 3DSSG~\cite{wald2020learning,wu2021scenegraphfusion} and ReplicaSSG~\cite{straub2019replica,hou2025fross} benchmarks, where our method significantly and consistently outperforms the strongest baseline. 
Overall, our \textbf{main contributions} are as follows:

\begin{itemize}
    \item A plug-and-play, training-free fusion framework for online 3D scene graph generation that propagates label and representation uncertainty throughout the merging pipeline without additional training.
    
    \item A probabilistic node association that jointly models 2D scene graph uncertainty and 3D representation uncertainty, replacing hard association gates with a principled likelihood-based formulation.
    
    \item A Dirichlet evidence accumulation scheme for both node semantics and edge labels, augmented with a class-conditional prior that completes edges for sparsely and never co-observed object pairs.

    \item State-of-the-art results on the 3DSSG and ReplicaSSG benchmarks, improving relationship Recall@1 by 18.1 points over the strongest online baseline on 3DSSG, while maintaining real-time latency at 15\,ms per frame.
\end{itemize}

% ============================================================
\section{Related Work}
\label{sec:related}
% ============================================================

\noindent\textbf{Offline 3D Scene Graph Generation.}
Early 3D scene graph methods treat the task as a global inference problem over a complete 3D point cloud. 
Wald~\etal~\cite{wald2020learning} introduced the 3DSSG dataset and a graph convolutional network~\cite{kipf2017semi} baseline that encodes PointNet~\cite{qi2017pointnet} features extracted from fully reconstructed meshes. Subsequent works further refine node or edge representations, yet operate identically on a pre-built point cloud graph. 
% Wu~\etal~\cite{wu2021scenegraphfusion} adds a feature-wise attention mechanism for relationship inference but still requires a segmented, globally consistent point cloud as input. 
Wang~\etal~\cite{wang2023vl} improve long-tail relation prediction via multi-modal knowledge distillation at training time, and Heo~\etal~\cite{heo2025objectcentric} design a discriminative object feature encoder with contrastive pretraining. 
These methods produce dense and accurate segmentation at the point cloud level, but are incompatible with real-time operation. 

\noindent\textbf{Online 2D-to-3D Scene Understanding.}
A parallel line of work projects 2D predictions into 3D incrementally. General online lifting approaches demonstrate that rich 3D representations can be built without offline reconstruction. For example, ConceptGraphs~\cite{gu2024conceptgraphs} fuses 2D foundation-model detections into a 3D object graph via multi-view association. More recent open-vocabulary methods~\cite{yin2024sg,tang2025onlineanyseg} extend this to semantic segmentation masks, merging 2D outputs into persistent 3D structures using semantic similarity and spatial overlap. 

In online 3D scene graph generation specifically, early methods rely on sparse point cloud reconstruction with SLAM technology~\cite{tateno2015real,campos2021orb}. SGFN~\cite{wu2021scenegraphfusion} and its predecessor MonoSSG~\cite{wu2023incremental} infer local frame graphs from incremental segmentation and merge them deterministically into a global structure. Feng~\etal~\cite{feng2025hyperrectangle} introduce hyperrectangle embeddings to debias relation prediction. 
Kim~\etal~\cite{kim20193} build a sparse semantic graph from RGB-D streams without running reconstruction, but rely on hard association rules. More recently, FROSS~\cite{hou2025fross} pairs a real-time 2D SGG model with Gaussian back-projection for 3D node representation. This methodology achieves faster-than-real-time operation by removing the dependency on SLAM reconstruction, but the 2D-to-3D merging is fully deterministic.
SCRSSG~\cite{yeo2025statistical} pursue a post-hoc confidence rescoring strategy that re-weights predictions using node-to-node and node-to-edge co-occurrence statistics from a global graph. This method focuses on alleviating prediction imbalance, and does not address the three sources of uncertainty identified in this work.
All the above methods adopt deterministic online fusion or global message passing, which are not uncertainty-aware.

\noindent\textbf{Uncertainty-Aware 3D Perception and Scene Graphs. }
Bayesian filtering has a long history in sensor fusion. The Joint Probabilistic Data Association Filter~\cite{rasmussen2001probabilistic,bar2009probabilistic,rezatofighi2015joint} computes per-observation marginal association probabilities across all candidate tracks simultaneously, replacing hard data association with a principled probabilistic assignment. More recent research extends this method for joint landmark and pose estimation in SLAM~\cite{bowman2017probabilistic,strecke2019fusion}, and track-let prediction in multi-target tracking~\cite{chiu2021probabilistic,saleh2021probabilistic}. These techniques address spatial representation uncertainty in tracking but do not model semantic class distributions or graph-structured relationships.

In the 2D scene graph community, several works have recognized that relationship prediction is inherently uncertain. Yang~\etal~\cite{yang2021probabilistic} propose a plug-and-play Probabilistic Uncertainty Modeling module that represents each subject-object union region as a Gaussian distribution rather than a deterministic feature vector, enabling diverse and semantically ambiguous predictions. Similarly, Li~\etal~\cite{li2023uncertainty} adopt Bayesian classifier reparameterization to handle annotation noise in relationship labels. These approaches improve 2D prediction quality by preserving distribution-level information during model training, yet they address uncertainty only within the 2D inference stage. 

Our work bridges these two lines. Inspired by the literature on probabilistic data association, we replace hard association gates with a probabilistic node association formulation that jointly models semantic and spatial uncertainty. We then preserve the full softmax output of the 2D SGG model throughout the fusion process via Dirichlet evidence accumulation, propagating prediction uncertainty directly into the 3D scene graph rather than discarding it at the 2D-to-3D boundary.
We adopt voxel- and Gaussian-based 3D representations from state-of-the-art online methods~\cite{tang2025onlineanyseg,hou2025fross} and demonstrate that our PUF is effective on both, and is thus generalizable across 3D representations.

% ============================================================
\section{Method}
\label{sec:method}
% ============================================================

\subsection{Problem Definition}

We formulate online 3D scene graph generation as the incremental construction of a directed graph $\mathcal{G}=(\mathcal{V},\mathcal{E})$,
where each node $v_i\in\mathcal{V}$ represents a 3D object with semantic label $c_i\in\{1,\ldots,C\}$ and 3D position $\boldsymbol{\mu}_i$, and each directed edge $(i,j)\in\mathcal{E}$ carries a relationship label
$r_{ij}\in\{1,\ldots,R\}$.
Given a stream of RGB-D frames
$\{(\mathbf{I}_t,\mathbf{D}_t,\mathbf{T}_t)\}_{t=1}^{T}$, where
$\mathbf{T}_t$ is the known camera pose, the goal is to build
$\mathcal{G}$ incrementally without retaining the full video history.

% ============================================================
\subsection{Framework Overview}
% ============================================================
\vspace{-3mm}
\begin{figure}[htbp!]
    \centering
    \includegraphics[width=\linewidth]{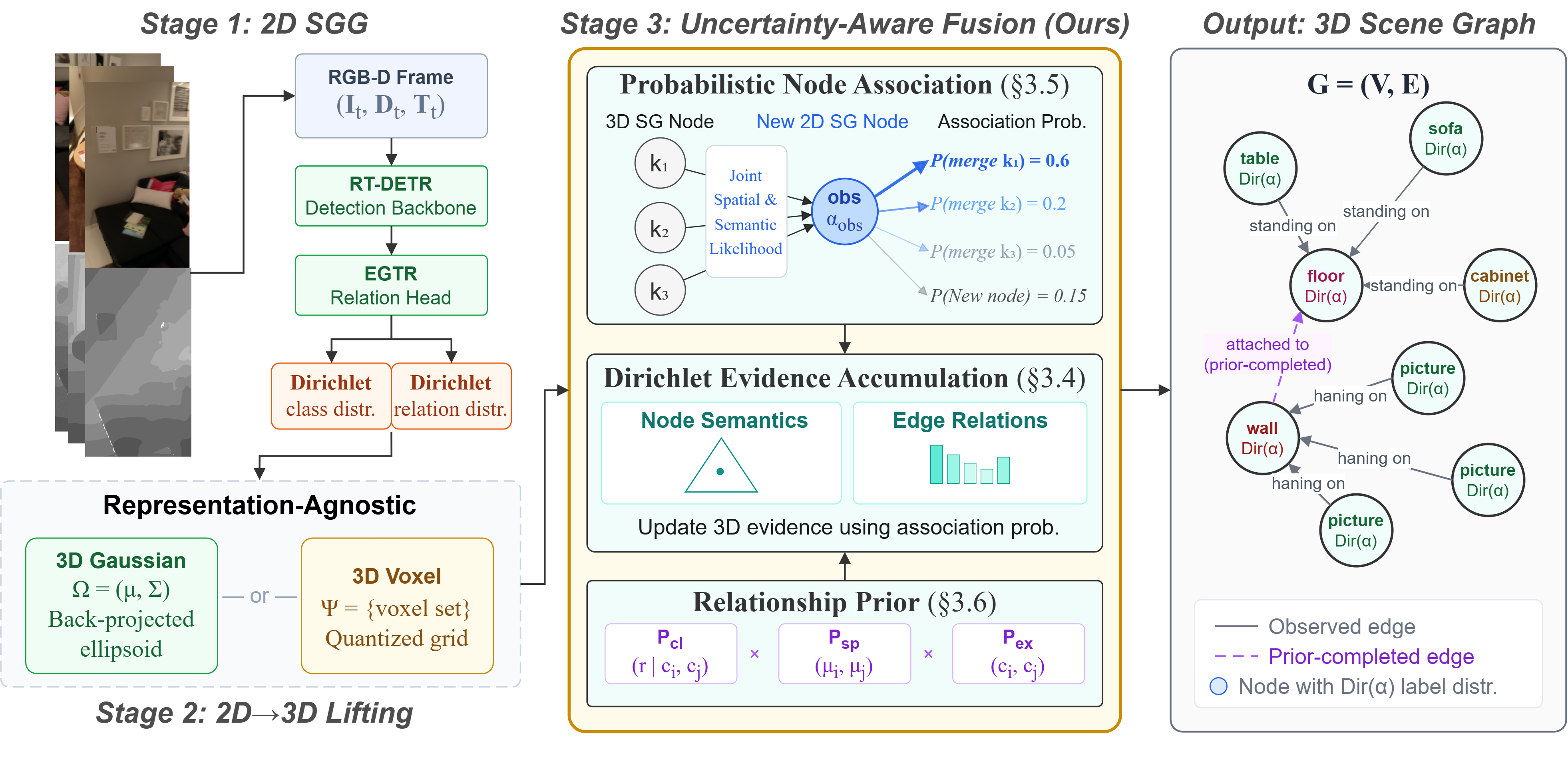}
    \caption{
        Overview of our proposed PUF. \textbf{Stage~1:} An RGB-D frame is processed by a 2D SGG model. We use Dirichlet-parameterization for object and relation label distributions.
        \textbf{Stage~2:} Each detection is
        lifted to 3D via a Gaussian or voxel backend (representation-agnostic). 
        \textbf{Stage~3:} Our uncertainty-aware fusion stage performs probabilistic node association and Dirichlet evidence accumulation, with an optional relationship prior for sparsely observed edges.
        \textbf{Output} 3D scene graph has nodes and edges carry full Dirichlet distributions over their respective label spaces.%
    }
    \label{fig:framework}
    \vspace{-3mm} %-3
\end{figure}

Figure~\ref{fig:framework} illustrates our uncertainty-aware fusion framework. At each frame, a 2D SGG model produces soft class and relation distributions, which are lifted to 3D observation nodes via either a Gaussian or voxel backend.
We use the Dirichlet distribution to accumulate object and relationship evidence on the global 3D SG (Sec.~\ref{subsec:dirichlet}).
Probabilistic node association (Sec.~\ref{subsec:association}) computes a joint association likelihood between a new observation and all existing global nodes, which replaces hard binary gates, and redistributes 2D evidence to 3D accumulators.
An optional relationship prior (Sec.~\ref{subsec:prior}) completes sparsely observed edges.

Our framework is \textit{plug-and-play}: it wraps any 2D SGG model that outputs soft class and relation distributions without modifying model architecture; and is \textit{training-free}: it requires no additional neural-network training. 

% ============================================================
\subsection{Preliminaries}
% ============================================================

\noindent\textbf{2D scene graph model.}
We use RT-DETR detection backbone~\cite{carion2020end,lv2024rt} with EGTR~\cite{im2024egtr} for 2D SGG, following FROSS~\cite{hou2025fross}. RT-DETR produces object bounding boxes and feature representations at real-time latency. EGTR predicts pairwise relationship logits via an edge-conditioned transformer.

PUF departs from existing hard-fusion methods by consuming the \textit{full} softmax outputs: the per-object class probability $\hat{p}^c\in\Delta^{C-1}$ (the ($C$-1)-simplex) and the per-pair relation probability $\hat{p}^r\in\Delta^{R-1}$, rather than only their argmax. In this way, the prediction uncertainty is propagated into the fusion layer.

\noindent\textbf{3D representation — Gaussian.}
Each detected bounding box is back projected to a 3D Gaussian $\Omega_i=(\boldsymbol{\mu}_i,\boldsymbol{\Sigma}_i)$ using per-pixel depth and the camera intrinsics/extrinsics.
The mean $\boldsymbol{\mu}_i$ is the depth-weighted centroid of the back-projected pixel cloud inside the box; the covariance $\boldsymbol{\Sigma}_i$ is the empirical covariance of the cloud.
Implementation is identical to FROSS~\cite{hou2025fross} for fair comparison; the details are provided in Supplementary Sec.~B.

\noindent\textbf{3D representation — Voxel.}
Each bounding box is converted to a discrete voxel set, inspired by OnlineAnySeg~\cite{tang2025onlineanyseg}. Depth-filtered pixels are first back-projected to the world coordinates and then quantized to a uniform grid of resolution $\delta$.
Depth filtering retains only pixels whose depth deviates by at most $\varepsilon_d$ from the box-centre depth, isolating foreground from background clutter. The full lifting and merging procedure is detailed in Supplementary Sec.~B.

% ============================================================
\subsection{Dirichlet Representation for Nodes and Edges}
\label{subsec:dirichlet}
% ============================================================

We use the Dirichlet distribution to represent the semantic labels of SG nodes and edges. It naturally replaces one-hot based label with a distribution over all categories, therefore carries 2D SGG model uncertainty. Its posterior update is straightforward and facilitates the evidence accumulation in the merging step. 

\noindent\textbf{Node class model.}
The class probability $\hat{p}^c_k$ of global 3D node $\mathcal{V}_k$ is modeled by a Dirichlet:
\begin{equation}
\hat{p}^c_k\sim\operatorname{Dir}(\boldsymbol{\alpha}_{k}).
\label{eq:node_model}
\end{equation}
where $\boldsymbol{\alpha}_k \in \mathbb{R}^C$ is the evidence accumulator for $\mathcal{V}_k$ over the semantic class label space, and interpreted as the concentration parameters of the Dirichlet distribution. Note that $\boldsymbol{\alpha}_k$ is real-valued and not normalized to facilitate association update (Eq.~\eqref{eq:dirichlet_update}).
The normalized Dirichlet posterior mean $\bar{\boldsymbol{\alpha}}_k = \boldsymbol{\alpha}_k / \|\boldsymbol{\alpha}_k\|_1$ gives a proper distribution over class labels, and the hard class prediction for the final output is $c_k = \arg\max\,\bar{\boldsymbol{\alpha}}_k$.

\noindent\textbf{Edge relation model.}
Similarly to nodes, for each directed edge $\mathcal{E}_{ij}$, we model the relationship class probability $\hat{p}^r_{ij}$ as:
\begin{equation}
\hat{p}^r_{ij}\sim\operatorname{Dir}(\boldsymbol{\phi}_{ij}).
\label{eq:rel_model}
\end{equation}
where $\boldsymbol{\phi}_{ij}\in \mathbb{R}^R$ is the accumulator for edge $\mathcal{E}_{ij}$. 
The final predicted relationship is the posterior mean argmax: $r_{ij}=\arg\max\,\bar{\boldsymbol{\phi}}_{ij}$.

% ============================================================
\subsection{Probabilistic Node Association}
\label{subsec:association}
% ============================================================

\noindent\textbf{Association likelihood.}
Between each pair of newly observed node $obs$ and an existing node $k$ in global 3D SG, we compute a factored likelihood:
\begin{equation}
L[obs,\,k] = L_{\text{sp}}(obs,\,k)\;\cdot\;L_{\text{se}}(obs,\,k).
\label{eq:likelihood}
\end{equation}
The spatial factor, unlike hard thresholds in existing methods, enters the likelihood as a continuous term that modulates association probabilities, and preserves as well as propagates 3D representation uncertainty into fusion. Its specific form depends on the representation:
\begin{equation}
L_{\text{sp}}(obs,\,k) = \begin{cases}
  BC\!\left(\Omega_{obs},\Omega_k\right)
    & \text{(Gaussian)},  \\[4pt]
  |\Psi_{obs}\cap\Psi_k|\;/\;|\Psi_{obs}|
    & \text{(Voxel)}.
\end{cases}
\label{eq:spatial}
\end{equation}
For the Gaussian backend, BC denotes the Bhattacharyya coefficient and measures spatial overlap between two Gaussians, and $BC = 1-H^2$ where $H$ is the Hellinger distance \cite{murrugarra2024probabilistic}.
For the voxel representation, we use the containment score as defined in Eq.~\eqref{eq:spatial} to measure spatial overlap between an observation with voxel set $\Psi_{obs}$ and a global node with accumulated set $\Psi_k$.

The semantic factor compares the Dirichlet posterior means of the observation and the global node via Jensen–Shannon divergence (JSD):
\begin{equation}
L_{\text{se}}(obs,\,k)
= \exp\!\left(-\,\frac{
    JSD\!\left(\hat{p}_{obs}^c,\,
                \bar{\boldsymbol{\alpha}}_k\right)}
  {\sigma_{se}}\right).
\label{eq:semantic}
\end{equation}
The bandwidth $\sigma_{se}$ controls semantic selectivity. A larger value increases the semantic factor and encourages more divergent $\left(\hat{p}_{obs}^c,\,\bar{\boldsymbol{\alpha}}_k\right)$ pairs to merge.

\noindent\textbf{Probabilistic marginals of node association.}
Inspired by the Joint Probabilistic Data Association Filter~\cite{bar2009probabilistic,rezatofighi2015joint}, we model the probability that the observation associates with global node $k$ as $\beta[obs,\,k]$. $\beta[obs,\,k]$ is obtained by marginalizing the likelihood vector $\{L[obs,\,k]\}_{k\in\mathcal{K}}$ together with a prior probability that $obs$ is a new object:
\begin{equation}
\beta[obs,\,k] = \frac{L[obs,\,k]}{Z},
\quad
\beta_{birth} = \frac{\lambda_{birth}}{Z},
\quad
Z = \lambda_{birth} + \textstyle\sum_{k\in\mathcal{K}} L[obs,\,k].
\label{eq:jpda}
\end{equation}
where $\lambda_{birth}$ is a birth term encoding the prior probability of observing a new object.
By construction $\beta_{birth}+\sum_k\beta[obs,\,k]=1$.
If $\beta_{birth}>0.5$, the observation spawns a new global node; otherwise it associates with existing nodes.
To ensure real-time computation we adopt per-observation independent normalization Eq.~\eqref{eq:jpda} rather than the expensive joint hypothesis enumeration as in the original Joint Probabilistic Data Association Filter~\cite{bar2009probabilistic}. This is also due to the fact that DETR's non-maximum suppression mechanism~\cite{lv2024rt} already prevents two detections of the same object within a single frame.
We provide a detailed analysis in Supplementary Sec.~A.

\noindent\textbf{Association update.}
If the new observation merges to the global 3D SG ($\beta_{birth}\leq0.5$), we apply semantic and spatial updates as follows:
\begin{alignat}{2}
&\text{Semantic:}\quad
  &\boldsymbol{\alpha}_k &\;\leftarrow\;
    \boldsymbol{\alpha}_k + \beta[obs,\,k] \cdot \hat{p}_{obs}^c,
    \quad\forall\,k:\beta[obs,\,k]\geq\beta_{min},
    \label{eq:dirichlet_update}\\[4pt]
&\text{Spatial:}\quad
  &k^* &= \operatorname*{arg\,max}_k\;\beta[obs,\,k].
    \label{eq:argmax}
\end{alignat}
where $\beta_{min}$ is an evidence cut-off threshold. Class evidence across all candidates with non-negligible $\beta[obs,\,k]$ is distributed to the global node $k$, which lets node $k$'s Dirichlet accumulator absorb uncertainty across near-tie associations.
3D representation (Gaussian mean, covariance/voxel union) is updated only for $k^*$ to avoid contaminating or drifting the positions of non-target nodes: a predicted label can be 0.6 chair - 0.4 table, but the object is not part chair, part table, so semantics are softly merged while spatial occupation is not. 

A soft relationship transfer is applied to the SG edges. 
Let $\hat{p}_{obs,j}^r$ denote the 2D relation prediction between the observation node and global node $j$ in the current frame.
The relationship evidence is redistributed to the same candidates and with the same weights as the semantic update:
\begin{equation}
\boldsymbol{\phi}_{kj}
  \;\leftarrow\; \boldsymbol{\phi}_{kj}
    + \beta[obs,\,k]\cdot\hat{p}_{obs,j}^r,
\quad
\boldsymbol{\phi}_{ik}
  \;\leftarrow\; \boldsymbol{\phi}_{ik}
    + \beta[obs,\,k]\cdot\hat{p}_{i,obs}^r,
\label{eq:rel_transfer}
\end{equation}
where the first term covers the observation acting as subject and the second as object. This update preserves consistency between the class and relation accumulators, giving evidence accumulation to poorly observed nodes and edges.
After $obs$ is merged to existing nodes, it is invalidated. 

Otherwise, if $\beta_{birth}>0.5$, then node $obs$ is a newly observed object and initializes a new 3D node from the soft 2D outputs:
\begin{gather}
\boldsymbol{\alpha}_{obs} \leftarrow \hat{p}_{obs}^c . \\
\boldsymbol{\phi}_{i,obs} \leftarrow \hat{p}_{i,obs}^r\,, \quad
\boldsymbol{\phi}_{obs,j} \leftarrow \hat{p}_{obs,j}^r\,.
\label{eq:node_init}
\end{gather}
Then $obs$ is appended to the list of global 3D nodes.

% ============================================================
\subsection{Relationship Prior for Edge Evidence Enhancement}
\label{subsec:prior}
% ============================================================

An edge in the SG is only well-observed if both of its associated nodes are well-observed, making edges more susceptible to observation noise. To address this, we leverage Dirichlet conjugacy and propose an informative prior. For a pair with object classes $c_i,c_j$ and 3D centroids $\boldsymbol{\mu}_i,\boldsymbol{\mu}_j$, the prior factorizes as:
\begin{equation}
\boldsymbol{\phi}^{prior}_{ij}
= P_{cl}(r\mid c_i,c_j)
  \cdot P_{sp}(\boldsymbol{\mu}_i,\boldsymbol{\mu}_j)
  \cdot P_{ex}(c_i,c_j).
\label{eq:prior}
\end{equation}
\emph{Class-conditional prior $P_{cl}(r\mid c_i,c_j)$} is precomputed from training-set annotation counts with Laplace smoothing:
\begin{equation}
P_{cl}(r\mid c_i,c_j)
= \frac{n_{ij}^r+\varepsilon}{\sum_{r'}n_{ij}^{r'}+R\varepsilon},
\label{eq:pclass}
\end{equation}
where $n_{ij}^r$ counts relation $r$ between object-class pair $(c_i,c_j)$ in training data, and Laplace smoothing $\varepsilon=0.1$. This requires a single offline pass over annotations without gradient-based training.

\emph{Spatial prior $P_{sp}(\boldsymbol{\mu}_i,\boldsymbol{\mu}_j)$} is computed online from 3D positions, decaying the prior as spatial distance increases: $P_{sp}(\boldsymbol{\mu}_i,\boldsymbol{\mu}_j) = e^{-\frac{\mid\boldsymbol{\mu}_i -\boldsymbol{\mu}_j\mid}{2}}$. The combination of the class and spatial priors is normalized before multiplying with $P_{ex}(c_i,c_j)$.

\emph{Existence gate $P_{ex}(c_i,c_j)$} is the fraction of training-set pairs of class $(c_i,c_j)$ that carry any annotated relation. It scales the prior down for class pairs that rarely interact, suppressing false-positive completions for rarely co-occurring class pairs.

For observed edges, the final distribution is the Bayesian posterior:
\begin{equation}
\hat{p}^r_{ij}\sim\operatorname{Dir}(\boldsymbol{\phi}_{ij}^{post.}), \qquad
\boldsymbol{\phi}_{ij}^{post.}
= \boldsymbol{\phi}_{ij}^{prior}
  + \boldsymbol{\phi}_{ij},
\label{eq:posterior}
\end{equation}
where $\boldsymbol{\phi}_{ij}$ is the accumulated evidence by Eq.~\eqref{eq:rel_transfer}.  For unobserved edges, the prior alone is used if $\max(\boldsymbol{\phi}_{ij}^{prior})>0.5$. In the case no prior is used, the relationship label is predicted from Eq.\eqref{eq:rel_model}.

% ============================================================
\section{Experiments}
\label{sec:experiments}
% ============================================================

% ============================================================
\subsection{Experimental Setup}
% ============================================================

\subsubsection{Datasets.}

The \emph{3DSSG} dataset~\cite{wald2020learning} extends 3RScan~\cite{wald2019rio}, comprising 1,482 RGB-D indoor scans with dense instance segmentation and directed inter-object relationship annotations. We follow the standard protocol of~\cite{wald2020learning,wu2021scenegraphfusion,wu2023incremental,hou2025fross}, mapping 160 object categories to the 20 NYUv2 classes~\cite{silberman2012indoor} and retaining the 7 most frequent predicate types. After mapping, the dataset contains 21{,}974 annotated object instances and 16{,}324 inter-object relationships. We use the official training, validation, and test splits throughout.

\noindent\emph{ReplicaSSG}~\cite{hou2025fross} extends the Replica dataset~\cite{straub2019replica} with inter-object relationship annotations across 18 photorealistic indoor scenes. Its label space is adopted from Visual Genome~\cite{krishna2017visual} and contains 33 object and 9 relationship categories, facilitating zero-shot transfer from pretrained 2D SG models. Of the 18 scenes, 7 serve as a validation set and 11 as the test set; there is no training split.

\subsubsection{Evaluation Metrics.}

We adopt the recall-based protocol of~\cite{wu2023incremental,hou2025fross}.
\emph{Object recall} measures the fraction of ground-truth instances matched to a correctly classified node. 
\emph{Predicate recall} measures correct relationship classification among pairs whose endpoints are both detected.
\emph{Relationship recall} requires both endpoints to be correctly detected \emph{and} the predicate correctly classified. 
We additionally report \emph{mean recall} (mRecall), which averages per-class recall to mitigate the severe predicate class imbalance present in both datasets~\cite{wu2021scenegraphfusion,wu2023incremental,hou2025fross}.

\subsubsection{Implementation Details. }

We use EGTR~\cite{im2024egtr} with RT-DETRv2-M~\cite{lv2024rt} as the \emph{2D scene graph model}, following FROSS~\cite{hou2025fross} for fair comparison. For 3DSSG experiments, the model is trained on 2D scene graphs extracted from the 3DSSG training split. For ReplicaSSG experiments, we use a model trained on Visual Genome~\cite{krishna2017visual} with no additional fine-tuning on Replica data.
At inference, object detections below a confidence threshold of 0.7 are discarded and the top-10 pairwise relation scores per frame are retained, matching the FROSS configuration. All models are evaluated on an Intel XEON Gold 6534 processor with a single CPU core, and an NVIDIA L40S GPU.

\noindent\emph{3D representations.}
For the Gaussian backend, the back-projection and covariance estimation follow FROSS~\cite{hou2025fross} without modification.
For the voxel backend, we quantize depth-filtered back-projected points to a grid of $\delta=2$\,cm and retain only pixels whose depth deviates by at most $\varepsilon_d=0.3$\,m from the bounding box center depth for foreground isolation.

\noindent\emph{Hyperparameters.}
For Probabilistic Node Association and on both 3DSSG and ReplicaSSG, the semantic bandwidth is set to $\sigma_{se}=0.3$; the birth density is $\lambda_{birth}=0.4$; the minimum association weight for soft updates is $\beta_{min}=0.05$. 
For the Relationship Prior, the class-conditional prior $P_{cl}$ is precomputed offline from the 3DSSG training split, and no prior is used for ReplicaSSG since it does not have a training set. 
All hyperparameters are selected by grid search on the respective validation splits with relationship recall as the objective. Detailed analysis is provided in Section~\ref{subsec:ablation} for the 3D Gaussian representation and in Supplementary Sec.~D.2 for the voxel.

% ============================================================
\subsection{Quantitative Results and Comparison}
% ============================================================

Tables~\ref{table:main_3dssg} and~\ref{table:main_replica} present performance and latency comparisons on the 3DSSG and ReplicaSSG benchmarks, respectively. On 3DSSG, our PUF-Gaussian variant achieves the highest relationship recall of 46.0\%, improving over FROSS by 18.1 points while operating at 15\,ms per frame. Compared to the SGG recall gain, the additional fusion latency is negligible, and our method remains faster than all RGB-D+SLAM baselines by an order of magnitude. This further substantiates the advantages of lifting scene graphs from 2D images over direct point cloud reasoning. 
On ReplicaSSG, where no relationship prior is used due to the absence of a training split, PUF-Gaussian variant consistently outperforms FROSS across all metrics. This demonstrates that our proposed fusion framework itself provides reliable gains, independently of the prior. This is further supported by our ablation studies Tab.~\ref{table:ablation_3dssg} and ~\ref{table:ablation_replica} in Sec.~\ref{subsec:ablation}. Our Voxel variant further demonstrates that the improvements are representation-agnostic: consistent gains over FROSS are observed with both backends across both benchmarks, confirming that the framework generalizes beyond any single 3D representation choice.

\begin{table}[!htbp]
    \centering
    \caption{
        Performance comparison of 3D SGG methods using different input modality on the 3DSSG test set.  The end-to-end latency for baseline methods are obtained without the environmental mapping step. The best and second-best results are highlighted in \textbf{bold}, and \underline{underline}, respectively. Latencies are reported in milliseconds.
    }
    \small
    \resizebox{0.9\textwidth}{!}{
        \Large
        \begin{tabular}{l|ccc|cc|c|c}
        \toprule
         & \multicolumn{3}{c|}{\textbf{Recall (\%)}$\uparrow$} & \multicolumn{2}{c|}{\textbf{mRecall (\%)}$\uparrow$} & & \multirow{2}{*}{\shortstack{\textbf{Input}\\\textbf{Modality}}} \\
        \textbf{Method} & \textbf{Rel.} & \textbf{Obj.} & \textbf{Pred.} & \textbf{Obj.} & \textbf{Pred.} & \textbf{Latency}$\downarrow$ & \\
        \hline
        3DSSG~\cite{wald2020learning}       & 12.9 & 37.4 & 22.0 & 26.2 & 14.4 & -   & Point Cloud \\
        VL-SAT~\cite{wang2023vl}      & 23.5 & 53.7 & 28.9 & 42.1 & 25.3 & -   & Point Cloud \\
        OCRL~\cite{heo2025objectcentric}        & 25.2 & 58.5 & 30.1 & 49.6 & 27.1 & -   & Point Cloud \\
        SGFN~\cite{wu2021scenegraphfusion}        & 22.0 & 51.6 & 27.5 & 37.7 & 24.0 & 245 & RGB-D+SLAM \\
        MonoSSG~\cite{wu2023incremental}     & 23.3 & 53.8 & 28.4 & 43.8 & 26.6 & 283 & RGB-D+SLAM \\
        SCRSSG~\cite{yeo2025statistical}      & 25.7 & 61.1 & 27.6 & 60.5 & \underline{27.8} & 350 & RGB-D+SLAM \\
        VGfM~\cite{gay2018visual}        & 19.6 & 50.0 & 20.4 & 34.8 & 11.0 & 379 & RGB \\
        IMP~\cite{xu2017scene}         & 19.7 & 49.5 & 20.9 & 34.7 & 13.8 & -   & RGB-D \\
        Kim \etal~\cite{kim20193}   &  9.1 & 59.0 &  7.1 & 51.0 &  8.0 & 454 & RGB-D \\
        FROSS~\cite{hou2025fross}       & 27.9 & 62.5 & 33.2 & 63.8 & 18.1 & \textbf{13}   & RGB-D \\
        \midrule
        PUF-Voxel    & \underline{40.3} & \underline{65.5} & \underline{46.1} & \underline{64.1} & 21.8 & 31 & RGB-D \\
        PUF-Gaussian & \textbf{46.0} & \textbf{69.7} & \textbf{51.4} & \textbf{65.8} & \textbf{28.2} & \underline{15} & RGB-D \\
        \bottomrule
        \end{tabular}
    }
    \label{table:main_3dssg}
\end{table}

\begin{table}[!htbp]
    \centering
    \caption{
        Performance comparison of 3D SGG methods on the ReplicaSSG test set and the end-to-end latency. The best and second-best results are highlighted in \textbf{bold}, and \underline{underline}, respectively. Latencies are reported in milliseconds.
    }
    \resizebox{0.7\textwidth}{!}{
        \Large
        \begin{tabular}{l|ccc|cc|c}
        \toprule
         & \multicolumn{3}{c|}{\textbf{Recall (\%)}$\uparrow$} & \multicolumn{2}{c|}{\textbf{mRecall (\%)}$\uparrow$} & \multirow{2}{*}{\textbf{Latency}$\downarrow$} \\
        \textbf{Method} & \textbf{Rel.} & \textbf{Obj.} & \textbf{Pred.} & \textbf{Obj.} & \textbf{Pred.} & \\
        \hline
        FROSS~\cite{hou2025fross}  & \underline{22.5} & 26.2          & \underline{28.0} & 29.1          & \underline{20.6} & \textbf{14} \\
        PUF-Voxel    & 22.4             & \underline{27.9} & 26.9          & \underline{30.2} & 17.9          & 30 \\
        PUF-Gaussian & \textbf{25.3}    & \textbf{31.0} & \textbf{35.6}    & \textbf{33.7} & \textbf{26.2}    & \underline{16} \\
        \bottomrule
        \end{tabular}
    }
    \label{table:main_replica}
\end{table}

% ============================================================
\subsection{Ablation Study}
\label{subsec:ablation}
% ============================================================

\subsubsection{Module Effectiveness.}

Tables~\ref{table:ablation_3dssg} and~\ref{table:ablation_replica} report ablation results on 3DSSG and ReplicaSSG, respectively. We ablate three components: Dirichlet node representation, Dirichlet edge representation, and the relationship prior, across both the Gaussian and voxel backends. Probabilistic association is only enabled if the Dirichlet node representation is used.

\begin{table}[!htbp]
    \centering
    \caption{
        Ablation study of model components with recall and mean recall metrics on 3DSSG test set. Components Node and Edge means using Dirichlet representation for SG  nodes and edges, respectively. Latencies are reported in milliseconds.
    }
    \small
    \resizebox{\textwidth}{!}{
        \Large
        \begin{tabular}{l|ccc|ccc|cc|c|ccc|cc|c}
        \toprule
        \multirow{3}{*}{\textbf{Method}} & \multicolumn{3}{c|}{\textbf{Component}} & \multicolumn{6}{c|}{\textbf{Gaussian}} & \multicolumn{6}{c}{\textbf{Voxel}} \\
        & & & & \multicolumn{3}{c}{Recall@1$\uparrow$} & \multicolumn{2}{c|}{mRecall@1$\uparrow$} & \multirow{2}{*}{Latency$\downarrow$} & \multicolumn{3}{c}{Recall@1$\uparrow$} & \multicolumn{2}{c|}{mRecall@1$\uparrow$} & \multirow{2}{*}{Latency$\downarrow$} \\
        & Node & Edge & Prior & Rel. & Obj. & Pred. & Obj. & Pred. & & Rel. & Obj. & Pred. & Obj. & Pred. & \\
        \hline
        FROSS & & & & 27.9 & 62.4 & 33.2 & 63.8 & 18.1 & \textbf{13.2} & 23.7 & 61.6 & 29.5 & 62.6 & 16.5 & \textbf{28.9} \\
        \midrule
        \multirow{5}{*}{PUF}
        & \ding{51} & & & 31.3 & \underline{69.0} & 36.6 & \underline{65.8} & 20.9 & 14.6 & 28.6 & \underline{64.9} & 33.4 & \underline{64.1} & 17.8 & 31.0 \\
        & \ding{51} & \ding{51} & & 33.9 & \underline{69.0} & 39.2 & \underline{65.8} & 22.4 & 14.7 & 31.3 & \underline{64.9} & 35.2 & \underline{64.1} & 18.3 & 31.3 \\
        & & \ding{51} & & 27.9 & 62.9 & 33.7 & 64.3 & 21.9 & \underline{13.5} & 24.2 & 61.8 & 30.6 & 62.8 & \underline{18.0} & \underline{29.2} \\
        & & \ding{51} & \ding{51} & \underline{41.4} & 62.9 & \underline{47.8} & 64.3 & \underline{27.4} & \underline{13.5} & \underline{37.2} & 61.8 & \underline{40.0} & 62.8 & 19.6 & \underline{29.2} \\
        & \ding{51} & \ding{51} & \ding{51} & \textbf{46.0} & \textbf{69.7} & \textbf{51.5} & \textbf{65.8} & \textbf{28.2} & 14.8 & \textbf{40.3} & \textbf{65.5} & \textbf{46.1} & \textbf{64.1} & \textbf{21.8} & 31.3 \\
        \bottomrule
        \end{tabular}
    }
    \label{table:ablation_3dssg}
    \vspace{-2mm}
\end{table}

On 3DSSG, adding the Dirichlet node representation alone improves relationship recall by 3.4 points over FROSS under the Gaussian backend, confirming that propagating 2D class uncertainty into the fusion layer benefits node association and semantic accumulation. Incorporating Dirichlet edge representation and enabling probabilistic association further improves relationship recall to 33.9\%, as soft relational evidence is redistributed across plausible candidates rather than collapsed to hard votes. On top of the probabilistic fusion, the relationship prior additionally contributes a large gain: enabling it alongside the edge Dirichlet representation lifts relationship recall to 41.4\%, as it completes evidence for sparsely and never co-observed pairs. The full model combining all three components achieves 46.0\% relationship recall. Consistent trends are observed with the voxel backend, confirming that each component's contribution is representation-agnostic.

\begin{table}[!htbp]
    \centering
    \caption{
        Ablation study of model components with recall and mean recall metrics on ReplicaSSG test set. Components Node and Edge means using Dirichlet representation for SG  nodes and edges, respectively. Latencies are reported in milliseconds.
    }
    \resizebox{\textwidth}{!}{
        \Large
        \begin{tabular}{l|cc|ccc|cc|c|ccc|cc|c}
        \toprule
        \multirow{3}{*}{\textbf{Method}} & \multicolumn{2}{c|}{\textbf{Component}} & \multicolumn{6}{c|}{\textbf{Gaussian}} & \multicolumn{6}{c}{\textbf{Voxel}} \\
        & & & \multicolumn{3}{c}{Recall@1$\uparrow$} & \multicolumn{2}{c|}{mRecall@1$\uparrow$} & \multirow{2}{*}{Latency$\downarrow$} & \multicolumn{3}{c}{Recall@1$\uparrow$} & \multicolumn{2}{c|}{mRecall@1$\uparrow$} & \multirow{2}{*}{Latency$\downarrow$} \\
        & Node & Edge & Rel. & Obj. & Pred. & Obj. & Pred. & & Rel. & Obj. & Pred. & Obj. & Pred. & \\
        \hline
        FROSS & & & 22.5 & 26.2 & 28.0 & 29.1 & 20.6 & \textbf{13.8} & 18.6 & 21.3 & 21.1 & 23.4 & 13.5 & \textbf{28.8} \\
        \midrule
        \multirow{3}{*}{PUF}
        & \ding{51} & & \underline{23.1} & \textbf{31.0} & \underline{30.5} & \textbf{33.7} & 22.5 & 15.4 & \underline{21.7} & \textbf{27.9} & \underline{24.5} & \textbf{30.2} & 16.8 & 29.5 \\
        & & \ding{51} & 23.0 & 26.1 & 29.7 & 28.8 & \underline{23.7} & \underline{15.2} & 20.8 & 25.4 & 24.3 & 26.8 & \underline{16.9} & \underline{29.3} \\
        & \ding{51} & \ding{51} & \textbf{25.3} & \textbf{31.0} & \textbf{35.6} & \textbf{33.7} & \textbf{26.2} & 15.6 & \textbf{22.4} & \textbf{27.9} & \textbf{26.9} & \textbf{30.2} & \textbf{17.9} & 29.7 \\
        \bottomrule
        \end{tabular}
    }
    \label{table:ablation_replica}
    \vspace{-2mm}
\end{table}

On ReplicaSSG, where no relationship prior is used, the node and edge Dirichlet components each independently improve over FROSS, and combining them with the probabilistic association yields the best performance of 25.3\% relationship recall under the Gaussian backend. This consistent improvement in the absence of any prior corroborates that the core uncertainty-aware fusion framework is effective on its own, and that the gains observed on 3DSSG are not solely attributable to the prior. 

\subsubsection{Hyperparameter choices.}

Table~\ref{tab:hyperparams} reports the sensitivity of the two main hyperparameters on the 3DSSG (w/o prior) and ReplicaSSG validation sets under the Gaussian backend. The semantic bandwidth $\sigma_{\text{se}}$ controls the selectivity of the semantic likelihood factor (Eq.~\eqref{eq:semantic}): values below 0.3 over-penalize distributional differences and fragment semantically similar nodes, while values above 0.3 over-merge distinct objects, degrading both object and relationship recall. 
The birth density $\lambda_{\text{birth}}$ governs the prior probability of spawning a new global node (Eq.~\eqref{eq:jpda}). A large $\lambda_{\text{birth}}$ under-merges the graph, leaving object recall roughly intact, but depressing predicate and relationship recall as relational evidence is scattered across duplicate nodes. A small $\lambda_{\text{birth}}$ over-merges distinct objects, harming all metrics. Notably, the same setting $(\sigma_{\text{se}}{=}0.3,\;\lambda_{\text{birth}}{=}0.4)$ is optimal on both datasets, 
indicating that the framework is not sensitive to precise tuning. Voxel-backend hyperparameter analysis is provided in Supplementary Sec.~D.2 and exhibits consistent trends.

\begin{table}[htbp]
  \centering
  \caption{
  Performance comparison on (w/o prior) 3DSSG/ReplicaSSG validation sets with different semantic bandwidth $\sigma_{se}$ and birth density $\lambda_{birth}$. 
  }
  \resizebox{0.9\textwidth}{!}{
    \begin{tabular}{c|l|lllll|lllll}
    \toprule
    \multirow{2}[2]{*}{Dataset} & \multicolumn{1}{c|}{\multirow{2}[2]{*}{Recall(\%)$\uparrow$}} & \multicolumn{5}{c|}{$\sigma_{se}$} & \multicolumn{5}{c}{$\lambda_{birth}$} \\
          &       & \multicolumn{1}{c}{0.2} & \multicolumn{1}{c}{0.25} & \multicolumn{1}{c}{0.3} & \multicolumn{1}{c}{0.35} & \multicolumn{1}{c|}{0.4} & \multicolumn{1}{c}{0.3} & \multicolumn{1}{c}{0.35} & \multicolumn{1}{c}{0.4} & \multicolumn{1}{c}{0.45} & \multicolumn{1}{c}{0.5} \\
    \midrule
    \multirow{3}[2]{*}{3DSSG} & Rel.  & 32.4  & \underline{33.7}  & \textbf{34.2}  & 31.0  & 29.3  & 30.9  & 31.7  & \textbf{34.2}  & \textbf{34.2}  & \textbf{34.2} \\
          & Obj.  & 65.1  & \underline{66.5}  & \textbf{66.9}  & 65.8  & 64.0  & 65.3  & 66.0  & 66.9  & \underline{67.7}  & \textbf{68.1} \\
          & Pred. & 36.6  & \underline{37.4}  & \textbf{38.2}  & 35.7  & 34.6  & 34.8  & 35.6  & \textbf{38.2}  & \underline{37.8}  & 36.5 \\
    \midrule
    \multirow{3}[2]{*}{ReplicaSSG} & Rel.  & 16.1  & \underline{16.3}  & \textbf{17.6}  & 14.7  & 13.6   & 12.8   & 13.9  & \textbf{17.6}  & \underline{14.6}  & \underline{14.6} \\
          & Obj.  & 23.3  & \underline{24.5}  & \textbf{24.5}  & 23.6  & 22.2  & 21.3  & 21.7  & \textbf{24.5}  & \underline{23.0}  & \underline{23.0} \\
          & Pred. & 20.8  & \underline{22.0}  & \textbf{22.4}  & 19.9  & 18.3  & 19.0  & 19.4  & \textbf{22.4}  & \underline{20.3}  & 19.1 \\
    \bottomrule
    \end{tabular}%
    }
  \label{tab:hyperparams}%
\end{table}%

\subsubsection{Relationship Prior on 3DSSG. }

The 3DSSG dataset presents a notable co-visibility challenge: despite an average of ${\sim}245$ frames per scan, 25.3\% of ground-truth object pairs accumulate fewer than 10 joint frame observations, and 2.7\% never co-appear at all. 
Figure~\ref{fig:prior} stratifies relationship recall by the number of 2D co-observations per ground-truth triplet. 
Two findings stand out. First, even without the relationship prior, our uncertainty-aware fusion framework (``w/o Prior'') consistently outperforms FROSS, confirming that probabilistic node association and soft evidence redistribution alone improve recall for both well- and poorly-observed edges. This is further confirmed by a quantitative breakdown in Supplementary Sec.~C. 
Second, the relationship prior yields its largest gains in the sparsest bins, and its marginal contribution diminishes monotonically as co-observation count increases. This suggests that the prior behaves correctly: the prior fills in evidence that the observation stream cannot provide, while deferring to accumulated observations when they are plentiful.

\begin{figure}[htbp!]
    \centering
    \includegraphics[width=0.9\linewidth]{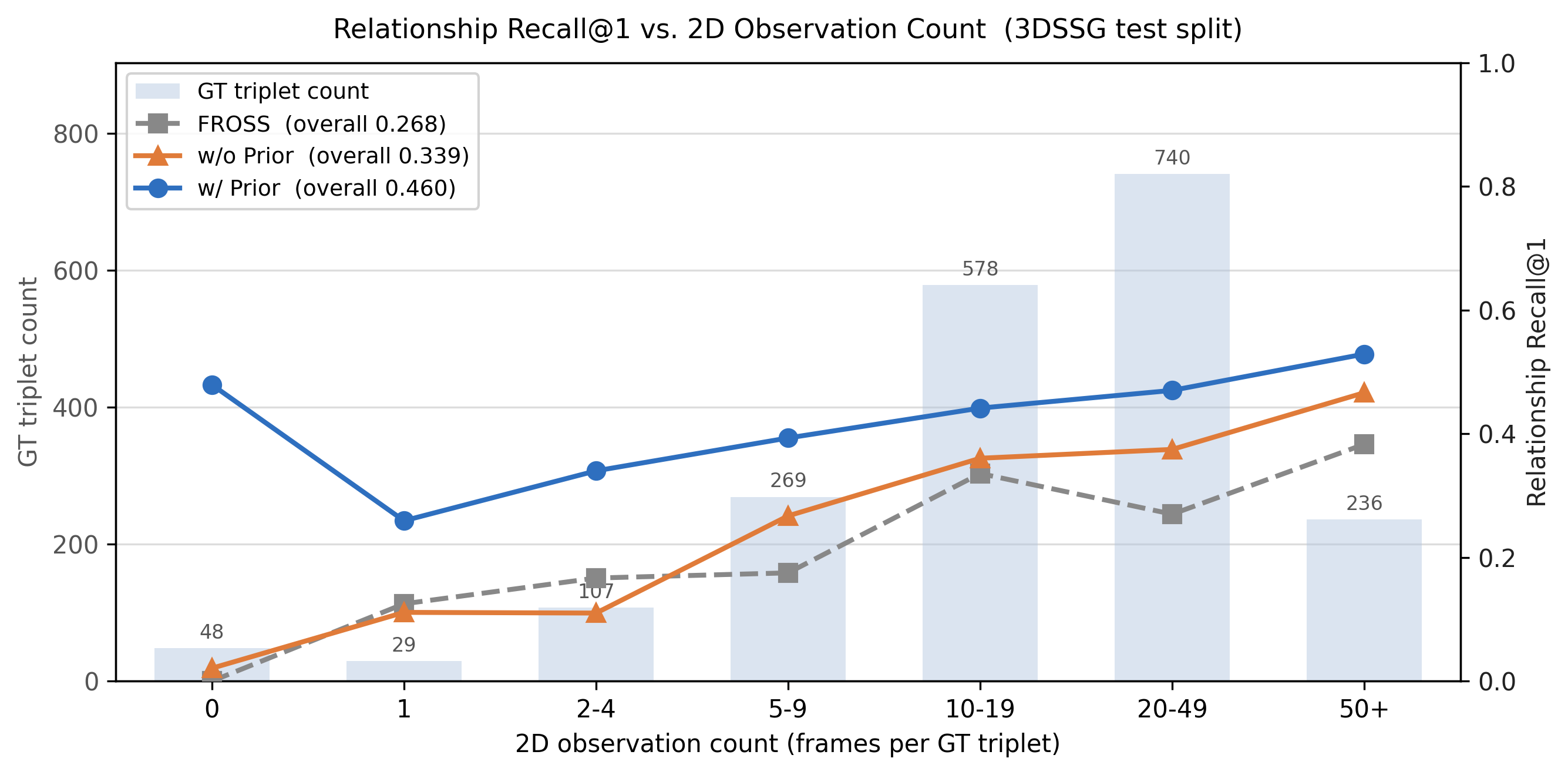}
    \caption{
        Relationship recall@1 under different 2D observation counts on 3DSSG test set. W/o prior denotes our full fusion framework without using relationship prior. 
    }
    \label{fig:prior}
    %\vspace{-3mm} %-3
\end{figure}

% ============================================================
\subsection{Qualitative Results}
% ============================================================

\begin{figure}[htbp!]
    \centering
    \includegraphics[width=\linewidth]{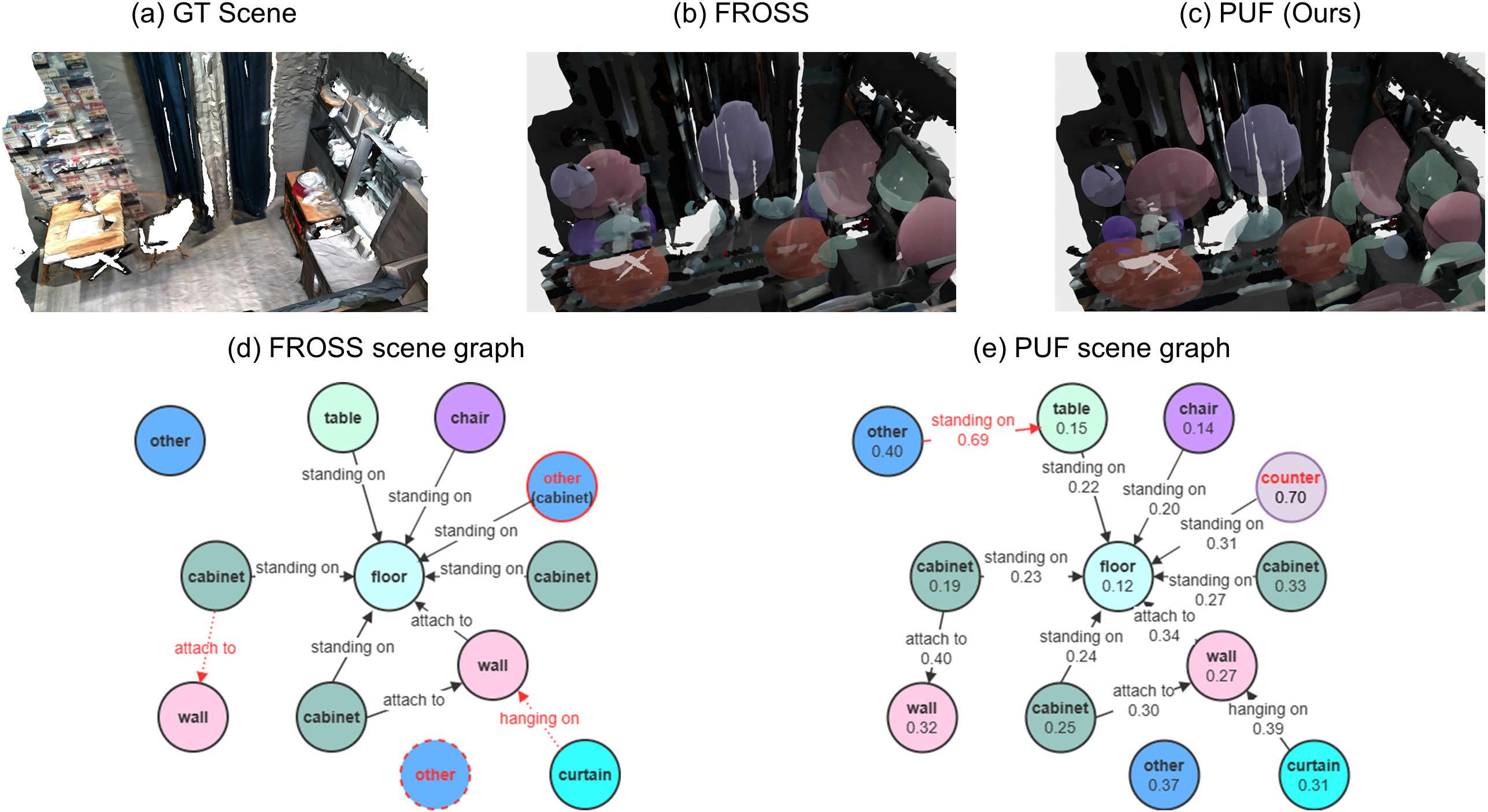}
    \caption{
        Qualitative between our PUF and FROSS~\cite{hou2025fross} on 3DSSG dataset with Gaussian 3D representation. Node colors correspond to the respective scene graphs. False and missing predictions are marked in \textcolor{red}{red}, with missing predictions in dashed line. 
    }
    \label{fig:vis}
\end{figure}

Figure~\ref{fig:vis} compares a representative scene from 3DSSG under FROSS and our method using the Gaussian backend. FROSS misclassifies several rare nodes in the \texttt{other} class (marked in red) and misses the rare predicate \texttt{hanging on} between \texttt{curtain} and \texttt{wall}, defaulting instead to the dominant \texttt{standing on}. Our method correctly recovers the object and the rare relationship, illustrating the benefit of soft evidence accumulation over hard label transfer. 
Additional qualitative outputs on ReplicaSSG and with the voxel backend are provided in Supplementary Sec.~E, which show consistent results.

Because the Dirichlet representation maintains a full distribution over labels, we can additionally visualize per-node and per-edge predictive uncertainty via the normalized entropy $H/H_{\max}$, where $H_{max}$ is the entropy over an even Dirichlet distribution. As shown in the right-hand graph of Fig.~\ref{fig:vis}, correctly classified nodes and edges exhibit low normalized entropy, while the few remaining errors, such as the \texttt{standing on} edge and the misclassified \texttt{counter}, carry notably higher uncertainty. This calibration property is a direct consequence of propagating the three sources of uncertainty identified in Sec.~\ref{sec:intro} through the fusion pipeline, and could be leveraged by downstream consumers to filter or re-query unreliable predictions.

% ============================================================
\section{Conclusion}
\label{sec:conclusion}
% ============================================================

We presented a plug-and-play, training-free fusion framework, \textbf{PUF}, for online 3D scene graph generation. Our framework replaces the deterministic merging pipelines of prior works with uncertainty-aware probabilistic association and Dirichlet evidence accumulation. By preserving 2D prediction distributions, modeling 3D representation uncertainty in node association, and optionally completing sparsely observed edges with a class-conditional prior, our framework yields substantial gains in relationship recall on both the 3DSSG and ReplicaSSG benchmarks while adding negligible latency. Its representation-agnostic design is validated with both Gaussian and voxel backends, and makes it readily applicable as a drop-in wrapper around any 2D scene graph model that outputs soft class and relation distributions.

% ============================================================

\section*{Acknowledgements}
This work has been supported by the 
Centre for Spatial Intelligence (RCSI) at University
of Bath,  the European Union under grant agreement no. 101136006-XTREME,
the European Innovation Council under grant agreement no. 101257536-CEREBRIS,
the MWK of Lower Saxony within Hybrint (VWZN4219) and LCIS (VWZN4704), the DFG under Germany’s Excellence Strategy within the Cluster of Excellence PhoenixD (EXC2122) and Quantum Frontiers (EXC2123).

% ============================================================
\section*{\centering Supplementary Materials}

\setcounter{section}{0}
\renewcommand{\thesection}{\Alph{section}}
% ---------------------------------------------------------------

% ============================================================
\section*{Overview}
\label{suppsec:intro}
% ============================================================

This supplementary document provides additional theoretical analysis, quantitative results, and qualitative visualizations that complement the main paper. The contents are organized as follows:

\begin{itemize}
    \item \textbf{Section~\ref{suppsec:norm}} provides a formal justification for the per-observation independent normalization used in our probabilistic node association. We show that under the $\epsilon$-separation condition enforced by the detector's non-maximum suppression, the factorized approximation we use recovers the exact Joint Probabilistic Data Association marginals with a bounded error of $O(NM\epsilon)$, which is negligible for typical frame sizes.
    \item \textbf{Section~\ref{suppsec:backend}} details the 2D-to-3D lifting and geometric merging of the 3D backends: the Gaussian representation adapted from FROSS and the voxel representation adapted from OnlineAnySeg, making the preliminaries of the main paper self-contained.
    \item \textbf{Section~\ref{suppsec:prior}} introduces a frequentist completion baseline that isolates the contribution of dataset statistics from PUF's probabilistic evidence accumulation, demonstrating that the Bayesian integration of prior evidence with accumulated observations is more effective than naive frequency-based edge completion.
    \item \textbf{Section~\ref{suppsec:per_class}} reports per-class object and predicate recall on both the 3DSSG and ReplicaSSG test sets, offering a fine-grained view of results.
    \item \textbf{Section~\ref{suppsec:hyper}} presents hyperparameter sensitivity analysis for the voxel backend on both datasets, complementing the Gaussian-backend analysis.
    \item \textbf{Section~\ref{suppsec:calibration}} shows that PUF is robust to temperature-scaled 2D model uncertainty, and remains plug-and-play to post-hoc calibration methods.
    \item \textbf{Section~\ref{suppsec:qualitative}}  provides additional qualitative comparisons between the Gaussian and voxel backends on 3DSSG and ReplicaSSG, respectively, illustrating that PUF produces consistent graph topology and interpretable uncertainty estimates across both 3D representations and datasets.
\end{itemize}

\newpage

% ============================================================
\section{Per-observation Normalization for Association}
\label{suppsec:norm}
% ============================================================

We provide a detailed analysis of how and why the Joint Probabilistic Data Association Filter (JPDAF) can be factorized into per-observation normalization under non-maximum suppression (NMS).

\subsection{Setup}

At frame $t$, the 2D detector produces $N$ bounding-box detections. Each detection is converted into a 3D Gaussian / voxel set and projected into the 3D space, forming the observation set:
\setcounter{equation}{14}
\begin{equation}
    \mathbf{Z} = \{z_1, \ldots, z_N\}
\end{equation}
Let $\mathcal{T} = \{1, \ldots, M\}$ denote the indices of the $M$ currently active global nodes in the 3D scene graph (\ie, \underline{tracks} in Joint Probability Data Association).
A \emph{joint association hypothesis} is a function:
\begin{equation}
    \theta : [N] \to \mathcal{T} \cup \{0\},
\end{equation}
that assigns each observation $i$ to either a new birth $\theta(i) = 0$ or to an existing global node $\theta(i) = k > 0$.

% ------------------------------------------------------------
\subsection{Exact Joint Posterior}
\label{sec:exact_jpdaf}
% ------------------------------------------------------------

\paragraph{Feasibility.}
A hypothesis $\theta$ is \emph{feasible} if each track emits at most one observation per frame:
\begin{equation}
    \Theta_{F}
    \;=\;
    \bigl\{\,\theta : [N] \to \mathcal{T} \cup \{0\}
    \;\big|\;
    \theta(i) = \theta(j) \neq 0 \;\Rightarrow\; i = j
    \,\bigr\}.
    \label{eq:feasible_set}
\end{equation}
$|\Theta_F|$ grows super-exponentially with $N$ and $M$: the number of feasible hypotheses equals $\sum_{k=0}^{\min(N,M)} \binom{N}{k}\binom{M}{k} k!\, 2^{N-k}$.
For typical frame sizes ($N \approx 20$, $M \approx 30$), this space of hypotheses is far too large to enumerate exhaustively in real time.

\paragraph{Likelihood model.}
Let $L [i, k] = L(z_i \mid k) \in (0, 1]$ denote the normalized likelihood that observation $i$ originates from global node $k$, and let $\lambda_b > 0$ denote the prior birth rate. The per-factor potential is defined as
\begin{equation}
    f_i(k)
    \;=\;
    \begin{cases}
        L [i, k] & k \in \mathcal{T}, \\[2pt]
        \lambda_b  & k = 0.
    \end{cases}
    \label{eq:potential}
\end{equation}
Under a Poisson birth model and conditionally independent observations given the scene state, the posterior over feasible joint hypotheses is
\begin{equation}
    P(\theta \mid \mathbf{Z})
    \;=\;
    \frac{1}{Z_F}
    \prod_{i=1}^{N} f_i\!\bigl(\theta(i)\bigr),
    \qquad
    Z_F
    \;=\;
    \sum_{\theta' \in \Theta_F}
    \prod_{i=1}^{N} f_i\!\bigl(\theta'(i)\bigr).
    \label{eq:joint_posterior}
\end{equation}

\paragraph{Exact JPDA marginals.}
The quantity of interest is the marginal association probability of observation $i$ with track $k$:
\begin{equation}
    \beta [i, k]
    \;=\;
    P\!\bigl(\theta(i) = k \mid \mathbf{Z}\bigr)
    \;=\;
    \sum_{\substack{\theta \in \Theta_F \\ \theta(i) = k}}
    P(\theta \mid \mathbf{Z})
    \;=\;
    \frac{L [i, k]}{Z_F}
    \underbrace{
        \sum_{\substack{\theta \in \Theta_F \\ \theta(i) = k}}
        \prod_{j \neq i} f_j\!\bigl(\theta(j)\bigr)
    }_{=:\; C_i(k)}.
    \label{eq:exact_marginal}
\end{equation}
The inner sum $C_i(k)$ requires marginalizing over all feasible completions of the remaining $N-1$ observations subject to the constraint that no other observation is assigned to $k$.
Computing $C_i(k)$ exactly is \#P-hard in general~\cite{mahler2007statistical}, as it is equivalent to computing the permanent of a $0/1$-weighted matrix.

% ------------------------------------------------------------
\subsection{Independent Per-Observation Approximation}
\label{sec:approx}
% ------------------------------------------------------------

\paragraph{Factorized approximation.}
Rather than summing over $\Theta_F$, our implementation drops the feasibility constraint and factorizes the joint distribution across observations:
\begin{equation}
    \hat{P}(\theta)
    \;=\;
    \prod_{i=1}^{N} \hat{P}_i\!\bigl(\theta(i)\bigr),
    \label{eq:factorised_joint}
\end{equation}
where each marginal is independently normalized:
\begin{equation}
    \hat{\beta}[i, k]
    \;=\;
    \hat{P}_i\!\bigl(\theta(i) = k\bigr)
    \;=\;
    \frac{L [i, k]}{\lambda_b + \displaystyle\sum_{k' \in \mathcal{T}} L[i, k']},
    \quad
    \hat{\beta}_i^{\,\mathrm{birth}}
    \;=\;
    \frac{\lambda_b}{\lambda_b + \displaystyle\sum_{k' \in \mathcal{T}} L[i, k']}.
    \label{eq:approx_marginal}
\end{equation}
The unconstrained partition function factorizes as
$Z = \prod_{i=1}^N \!\bigl(\lambda_b + \sum_{k} L [i, k]\bigr)$,
replacing the intractable sum over $\Theta_F$ with $N$ independent scalar normalizations.

% ------------------------------------------------------------
\subsection{Exactness Under Non-Maximum Suppression}
\label{sec:nms_exactness}
% ------------------------------------------------------------

We now formalize the condition under which the feasibility constraint in Eq.~\eqref{eq:joint_posterior} contributes negligible probability mass, so that the approximation Eq.~\eqref{eq:approx_marginal} recovers the exact marginals.

\begin{definition}[$\epsilon$-separation]
\label{def:separation}
Observation set $\mathbf{Z}$ is \emph{$\epsilon$-separated} with respect to $\mathcal{T}$ if, for every track $k \in \mathcal{T}$, at most one observation has non-negligible likelihood for $k$:
\begin{equation}
    \bigl|\bigl\{i \in [N] : L[i, k] > \epsilon\bigr\}\bigr| \;\leq\; 1
    \qquad \forall\, k \in \mathcal{T}.
    \label{eq:separation}
\end{equation}
\end{definition}

\begin{proposition}[Factorization under $\epsilon$-separation]
\label{prop:factorisation}
Let $\mathbf{Z}$ be $\epsilon$-separated with respect to $\mathcal{T}$.
Denote the unconstrained per-observation normalizer by
$Z_i = \lambda_b + \sum_k L[i, k] \geq \lambda_b > 0$.
Then the exact partition function satisfies
\begin{equation}
    Z - Z_F
    \;\leq\;
    \binom{N}{2} M \,\epsilon
    \prod_{i=1}^{N} Z_i,
    \label{eq:partition_bound}
\end{equation}
and the marginal approximation error satisfies
\begin{equation}
    \bigl|\beta[i, k] - \hat{\beta}[i, k]\bigr|
    \;\leq\;
    (N-1) M \,\epsilon
    \label{eq:marginal_bound}
\end{equation}
for every $i \in [N]$ and $k \in \mathcal{T}$.
\end{proposition}

\begin{proof}
\textbf{Step 1: Bound on infeasible mass.}
Let $\Theta_F^c = \Theta \setminus \Theta_F$ denote the set of infeasible hypotheses, \ie those in which at least two distinct observations share a track.
Every infeasible hypothesis contains at least one \emph{conflicting pair}: a pair $(i, j)$ with $i \neq j$ and $\theta(i) = \theta(j) = k$ for some $k \in \mathcal{T}$.
By a union bound over conflicting pairs and tracks:
\begin{equation}
    Z - Z_F
    \;=\;
    \sum_{\theta \in \Theta_F^c}
    \prod_{i} f_i\!\bigl(\theta(i)\bigr)
    \;\leq\;
    \sum_{\{i,j\} \subseteq [N]}\;
    \sum_{k \in \mathcal{T}}
    L[i, k]\, L[j, k]
    \prod_{\ell \notin \{i,j\}} Z_\ell.
    \label{eq:infeasible_mass}
\end{equation}
By $\epsilon$-separation, for every track $k$ and every pair $(i, j)$,
at most one of $L[i, k], L[j, k]$ exceeds $\epsilon$, and both are bounded
above by $1$, so $L[i, k] L[j, k] \leq \epsilon$.
Substituting into Eq.~\eqref{eq:infeasible_mass}:
\begin{equation}
    Z - Z_F
    \;\leq\;
    \binom{N}{2} M\, \epsilon \!\prod_{\ell \notin \{i,j\}} Z_\ell
    \;\leq\;
    \binom{N}{2} M\, \epsilon \prod_{i=1}^{N} Z_i,
    \label{eq:partition_bound_proof}
\end{equation}
establishing Eq.~\eqref{eq:partition_bound}.

\textbf{Step 2: Bound on marginal error.}
Fix observation $i$ and track $k$.
Let $\hat{C}_i(k) = \prod_{j \neq i} Z_j$ be the unconstrained sum over completions of the remaining $N-1$ observations (\ie $C_i(k)$ without the feasibility constraint).
The constrained sum $C_i(k)$ from Eq.~\eqref{eq:exact_marginal} differs from $\hat{C}_i(k)$ in two ways:
(a) each $j \neq i$ is forbidden from being assigned to $k$, contributing at most $L[j, k] \leq \epsilon$ per observation per track to the excluded mass;
and (b) any two remaining observations $j, j'$ are forbidden from sharing the same track, but by the argument of Step 1 this contributes $O(\epsilon^2)$.

Retaining only the leading-order term:
\begin{equation}
    \hat{C}_i(k) - C_i(k)
    \;\leq\;
    (N-1) M\, \epsilon \prod_{\ell \neq i} Z_\ell
    \;+\;
    O(\epsilon^2),
    \label{eq:completion_bound}
\end{equation}
where the $(N-1)M$ factor counts the observations-times-tracks pairs that contribute a term of order $\epsilon$.
Similarly, $Z_F \geq Z - \binom{N}{2} M \epsilon \prod_\ell Z_\ell$, so $Z_F \geq Z(1 - \binom{N}{2} M \epsilon)$ to first order.

Combining, the absolute error in the marginal is:
\begin{align}
    \bigl|\beta[i, k] - \hat{\beta}[i, k]\bigr|
    &= \left|
        \frac{L[i, k]\, C_i(k)}{Z_F}
        - \frac{L[i, k]\, \hat{C}_i(k)}{Z}
       \right| \nonumber\\
    &\leq
    L[i, k]\left(
        \frac{|\hat{C}_i(k) - C_i(k)|}{Z_F}
        + \hat{C}_i(k)\,\frac{|Z - Z_F|}{Z_F\, Z}
    \right) \nonumber\\
    &\leq
    (N-1)M\,\epsilon
    + O\!\left(\binom{N}{2} M \epsilon\right)
    \;=\;
    O(N M \epsilon),
    \label{eq:marginal_bound_proof}
\end{align}
where the last line uses $L[i, k] \leq 1$, $C_i(k)/Z_F \leq 1$, and
$\hat{C}_i(k)/Z = \prod_{j \neq i} Z_j / \prod_j Z_j$ $= 1/Z_i \leq 1/\lambda_b$.
This establishes Eq.~\eqref{eq:marginal_bound} to leading order in $\epsilon$.
\end{proof}

\begin{corollary}[Exact factorization at $\epsilon = 0$]
\label{cor:exact}
If $\mathbf{Z}$ is $0$-separated, i.e.\ every track has at most one observation with positive likelihood, then $\Theta_F^c$ has zero weight under Eq.~\eqref{eq:joint_posterior}, $Z_F = Z$, $C_i(k) = \hat{C}_i(k)$, and $\beta[i, k] = \hat{\beta}[i, k]$ exactly.
\end{corollary}

% ------------------------------------------------------------
\subsection{Application to EGTR with NMS}
\label{sec:nms_argument}
% ------------------------------------------------------------

\paragraph{NMS enforces $\epsilon$-separation in 2-D.}
EGTR~\cite{im2024egtr} applies class-conditioned non-maximum suppression before producing detections: any box $b_j$ with $\mathrm{IoU}(b_i, b_j) > \tau_{\mathrm{NMS}}$ and the same predicted class as a higher-scoring box $b_i$ is suppressed.
Each global node $k$ corresponds to a single physical object. If $k$ is visible in frame $t$, it projects to a compact image region whose extent is determined by the object's 3-D bounding volume and the camera intrinsics. Two distinct post-NMS detections of the same class cannot both overlap this region with $\mathrm{IoU} > \tau_{\mathrm{NMS}}$.
Since our likelihood $L[i, k] = L_{\mathrm{sp}}(i,k)\, L_{\mathrm{se}}(i,k)$ combines a spatial Bhattacharyya coefficient / containment score (which is small for non-overlapping Gaussians / voxels) and a class-conditional semantic term (which is near zero for different predicted classes), $\epsilon$-separation holds with
\begin{equation}
    \epsilon
    \;\lesssim\;
    \max_{i \neq j}\,
    \min\bigl(L[i, k],\, L[j, k]\bigr)
    \;\ll\; 1.
    \label{eq:epsilon_bound}
\end{equation}
Proposition~\ref{prop:factorisation} then guarantees that the independent per-observation normalization Eq.~\eqref{eq:approx_marginal} incurs a marginal error of at most $O(NM\epsilon)$, which for typical frame sizes ($N \approx 20$, $M \approx 30$) is negligible.

\paragraph{Relation to standard JPDAF.}
Classical JPDAF~\cite{bar2009probabilistic} is typically applied to radar or sonar tracking where clutter density is high and the
number of targets is moderate. In those settings, multiple detections of the same target are common and the feasibility constraint is necessary.
Our setting is essentially different. The detector's NMS post-processing reduces the per-frame miss-detection of the same object to negligible probability, making the feasibility constraint vacuous and the independent factorization exact in the $\epsilon \to 0$ limit established by Proposition~\ref{prop:factorisation}.

\newpage

% ============================================================
\section{Details of 2D-to-3D Lifting and 3D Merging}
\label{suppsec:backend}
% ============================================================

The main paper instantiates PUF with two 3D backends: a 3D Gaussian representation adapted from FROSS~\cite{hou2025fross} and a discrete voxel representation adapted from OnlineAnySeg~\cite{tang2025onlineanyseg}.
For completeness, this section provides the implementation details of (i) the 2D-to-3D \emph{lifting} that converts each 2D detection into a 3D primitive, and (ii) the geometric \emph{merging} that the two backends use to fuse a local observation into the global representation.

% ------------------------------------------------------------
\subsection{Gaussian Backend (FROSS)}
\label{suppsec:backend_gaussian}
% ------------------------------------------------------------

\paragraph{2D Gaussian representation of bounding boxes.}
Each detected object $v_i$ is first described by a 2D Gaussian $(\boldsymbol{\mu}_i^{\mathrm{2D}}, \boldsymbol{\Sigma}_i^{\mathrm{2D}})$ in the image plane.
Following FROSS, the bounding box is treated as a 2D \emph{uniform} distribution: the mean $\boldsymbol{\mu}_i^{\mathrm{2D}}=(p_x,p_y)^\top$ is the box center, and the covariance matches the second moment of a uniform box of width $W_i$ and height $H_i$,
\begin{equation}
    \boldsymbol{\Sigma}_i^{\mathrm{2D}}
    = \frac{1}{12}
    \begin{bmatrix}
        W_i^2 & 0 \\
        0 & H_i^2
    \end{bmatrix}.
    \label{eq:fross_2dcov}
\end{equation}

\paragraph{Back-projection to 3D.}
Given the depth $z$ at the box center and the camera intrinsics $K\in\mathbb{R}^{3\times3}$, rotation $R\in\mathbb{R}^{3\times3}$, and translation $\mathbf{t}\in\mathbb{R}^3$, the 3D mean is obtained by back-projecting the 2D center and transforming it to world coordinates,
\begin{equation}
    \boldsymbol{\mu}_i^{\mathrm{3D}}
    = R\,K^{-1}\,(p_x,p_y,1)^\top \cdot z + \mathbf{t}.
    \label{eq:fross_mean}
\end{equation}
The 2D covariance is lifted with the local affine (Jacobian) approximation of the perspective projection,
\begin{equation}
    J =
    \begin{bmatrix}
        \dfrac{f_x}{z} & 0 & -\dfrac{f_x p_x}{z^2} \\[8pt]
        0 & \dfrac{f_y}{z} & -\dfrac{f_y p_y}{z^2}
    \end{bmatrix},
    \label{eq:fross_jacobian}
\end{equation}
where $f_x,f_y$ are the focal lengths.
Because $J$ maps 3D to 2D, the covariance is back-projected through its pseudo-inverse $J^+$.
A naive inversion leaves the depth axis without variance, since the 2D box carries no depth-extent information; FROSS therefore injects a depth-axis variance equal to the average of the two recovered lateral variances, applied in the camera frame before the world rotation,
\begin{align}
    \boldsymbol{\Sigma}_i^{\mathrm{3D}\prime}
    &= J^+\,\boldsymbol{\Sigma}_i^{\mathrm{2D}}\,J^{+\top},
    \label{eq:fross_cov_pinv}\\[4pt]
    \boldsymbol{\Sigma}_i^{\mathrm{3D}\prime\prime}
    &= \boldsymbol{\Sigma}_i^{\mathrm{3D}\prime}
    +
    \begin{bmatrix}
        0 & 0 & 0 \\
        0 & 0 & 0 \\
        0 & 0 & \dfrac{(\boldsymbol{\Sigma}_i^{\mathrm{3D}\prime})_{1,1}+(\boldsymbol{\Sigma}_i^{\mathrm{3D}\prime})_{2,2}}{2}
    \end{bmatrix},
    \label{eq:fross_cov_depth}\\[4pt]
    \boldsymbol{\Sigma}_i^{\mathrm{3D}}
    &= R\,\boldsymbol{\Sigma}_i^{\mathrm{3D}\prime\prime}\,R^\top.
    \label{eq:fross_cov_world}
\end{align}
The resulting 3D Gaussian $\Omega_i=(\boldsymbol{\mu}_i,\boldsymbol{\Sigma}_i)$ with $\boldsymbol{\mu}_i=\boldsymbol{\mu}_i^{\mathrm{3D}}$ and $\boldsymbol{\Sigma}_i=\boldsymbol{\Sigma}_i^{\mathrm{3D}}$ is the observation node used throughout the fusion.

\paragraph{Gaussian merging.}
In FROSS, two nodes are fused when their Hellinger distance falls below a threshold $\delta_d$. For Gaussians $\mathcal{N}(\boldsymbol{\mu}_i,\boldsymbol{\Sigma}_i)$ and $\mathcal{N}(\boldsymbol{\mu}_j,\boldsymbol{\Sigma}_j)$, the Hellinger distance derives from the Bhattacharyya distance $B_D$,
\begin{align}
    H_D(i,j) &= \sqrt{1-\exp\!\bigl(-B_D(i,j)\bigr)},
    \label{eq:fross_hellinger}\\[4pt]
    B_D(i,j) &= \frac{1}{8}\,\Delta\boldsymbol{\mu}_{ij}^\top\,\boldsymbol{\Sigma}^{-1}\,\Delta\boldsymbol{\mu}_{ij}
    + \frac{1}{2}\ln\!\frac{\det\boldsymbol{\Sigma}}{\sqrt{\det\boldsymbol{\Sigma}_i\,\det\boldsymbol{\Sigma}_j}},
    \label{eq:fross_bhattacharyya}
\end{align}
with $\Delta\boldsymbol{\mu}_{ij}=\boldsymbol{\mu}_i-\boldsymbol{\mu}_j$ and $\boldsymbol{\Sigma}=\tfrac{1}{2}(\boldsymbol{\Sigma}_i+\boldsymbol{\Sigma}_j)$.
The matched Gaussians are fused by a confidence-weighted integration, where the weight $w_i$ of each node records its accumulated detection frequency across views, so that objects observed from more viewpoints contribute proportionally more,
\begin{align}
    \boldsymbol{\mu}_k &= \frac{w_i\boldsymbol{\mu}_i + w_j\boldsymbol{\mu}_j}{w_i+w_j},
    \label{eq:fross_mu_merge}\\[4pt]
    \boldsymbol{\Sigma}_k &= \frac{w_i\boldsymbol{\Sigma}_i + w_j\boldsymbol{\Sigma}_j}{w_i+w_j}
    + \frac{w_i w_j\,(\boldsymbol{\mu}_i-\boldsymbol{\mu}_j)(\boldsymbol{\mu}_i-\boldsymbol{\mu}_j)^\top}{(w_i+w_j)^2}.
    \label{eq:fross_sigma_merge}
\end{align}
The spatial likelihood $L_{\mathrm{sp}}$ of the main paper reuses the Bhattacharyya coefficient $BC=1-H_D^2=\exp(-B_D)$ from Eq.~\eqref{eq:fross_bhattacharyya} as a continuous overlap score, and the spatial update of node $k^*$ follows Eqs.~\eqref{eq:fross_mu_merge}--\eqref{eq:fross_sigma_merge}.

% ------------------------------------------------------------
\subsection{Voxel Backend (OnlineAnySeg)}
\label{suppsec:backend_voxel}
% ------------------------------------------------------------

\paragraph{Voxel lifting.}
For the voxel backend, each detected mask is lifted into 3D by back-projecting its depth-filtered pixels into world coordinates and quantizing them onto a uniform grid of resolution $\delta$, producing a voxel set $\Psi_{obs}$.
Depth filtering keeps only pixels whose depth deviates from the box-center depth by at most $\varepsilon_d$, which isolates the foreground object from background clutter falling inside the 2D box.

\paragraph{Voxel merging.}
OnlineAnySeg quantifies the spatial association between two voxel sets $V_a,V_b$ through an asymmetric \emph{overlap ratio}. Restricting to the voxels of $V_b$ that are visible from the frames $X(m_a)$ in which $m_a$ is observed, $\mathrm{Vis}(V_b,X(m_a))=\{v\in V_b\mid v\to X(m_a)\}$, the ratio is
\begin{equation}
    or_{(a,b)} = \frac{|V_a\cap V_b|}{|\mathrm{Vis}(V_b,X(m_a))|},
    \label{eq:anyseg_overlap}
\end{equation}
which measures how much of $m_b$ falls inside $m_a$ within their co-visible region. Two masks are merged when their overall similarity (combining overlap ratio with semantic and geometric feature similarity) exceeds a threshold, or when enough third-view masks support the association.
When merged, the new voxel set is the union of the constituent sets and the new weight is the sum of their individual weights,
\begin{equation}
    \Psi_k = \Psi_i \cup \Psi_j,
    \qquad
    w_k = w_i + w_j,
    \label{eq:anyseg_merge}
\end{equation}
and the global identity of every affected voxel is redirected to the merged node through an append-only mapping table.
PUF adopts the voxel-set union of Eq.~\eqref{eq:anyseg_merge} for the spatial update of $k^*$, and uses the simplified, single-frame containment score $|\Psi_{obs}\cap\Psi_k|/|\Psi_{obs}|$ defined in the main paper as the spatial likelihood $L_{\mathrm{sp}}$ in place of the visibility-restricted overlap ratio of Eq.~\eqref{eq:anyseg_overlap}, since the per-observation likelihood already operates on the current frame.

\newpage

% ============================================================
\section{Frequentist vs. Probabilistic Prior on 3DSSG}
\label{suppsec:prior}
% ============================================================

The frequentist completion baseline is designed to isolate the contribution of the dataset statistics in PUF's relationship prior by testing how well a simple co-occurrence prior alone can recover unobserved scene graph edges. 
The frequentist baseline applies the same training-set co-occurrence statistics as PUF but assigns them as hard completions to unobserved edges, without Dirichlet evidence accumulation.
Concretely, for each unobserved pair $(i, j)$, it looks up the empirical co-occurrence probability $P_{ex}(r\mid c_i,c_j)$. If this exceeds the same completion threshold of 0.5 as PUF, the relation distribution $P_{cl}(r\mid c_i,c_j)\cdot P_{sp}(\boldsymbol{\mu}_i,\boldsymbol{\mu}_j)$ is assigned as the prediction.
This contrasts directly with PUF's probabilistic relationship prior. The frequentist baseline therefore separates the effect of injecting dataset statistics, making any performance gap between it and PUF attributable specifically to the probabilistic evidence redistribution and accumulation in the latter.

\setcounter{table}{5}
\begin{table}[htbp]
  \centering
  \caption{Performance comparison using different forms of prior on 3DSSG.}
    \begin{tabular}{c|c|ccc|cc}
    \toprule
    \multicolumn{2}{c|}{\multirow{2}[2]{*}{Method}} & \multicolumn{3}{c|}{Recall@1(\%)$\uparrow$} & \multicolumn{2}{c}{mRecall@1(\%)$\uparrow$} \\
    \multicolumn{2}{c|}{} & \multicolumn{1}{c}{Rel.} & \multicolumn{1}{c}{Obj.} & \multicolumn{1}{c|}{Pred.} & \multicolumn{1}{c}{Obj.} & \multicolumn{1}{c}{Pred.} \\
    \midrule
    \multirow{4}[2]{*}{w/o prior} & FROSS-Voxel & 23.7  & 61.6  & 29.5  & 62.6  & 16.5 \\
          & PUF-Voxel & 31.3  & 64.9  & 35.2  & 64.1  & 18.3 \\
          & FROSS-Gaussian & 27.9  & 62.4  & 33.0  & 63.8  & 18.0 \\
          & PUF-Gaussian & 33.9  & \underline{69.0}  & 39.2  & \underline{65.8}  & 22.4 \\
    \midrule
    \multirow{4}[2]{*}{Freq. prior} & FROSS-Voxel & 30.4  & 61.6  & 37.8  & 62.6  & 19.7 \\
          & PUF-Voxel & 38.2  & 64.9  & 44.3  & 64.1  & 20.2 \\
          & FROSS-Gaussian & 36.7  & 62.4  & 44.6  & 63.8  & 26.4 \\
          & PUF-Gaussian & \underline{43.1}  & \underline{69.0}  & \underline{48.2}  & \underline{65.8}  & \underline{27.3} \\
    \midrule
    \multirow{2}[2]{*}{Prob. Prior} & PUF-Voxel & 40.3  & 65.5  & 46.1  & 64.1  & 21.8 \\
          & PUF-Gaussian & \textbf{46.0}  & \textbf{69.7}  & \textbf{51.4}  & \textbf{65.8}  & \textbf{28.2} \\
    \bottomrule
    \end{tabular}%
  \label{supptab:prior}%
\end{table}%

Table~\ref{supptab:prior} compares three edge-completion strategies on the 3DSSG test set: no prior, a frequentist prior, and our probabilistic prior. Two trends are evident. First, PUF without any prior already outperforms original FROSS, and is comparable to FROSS with the frequentist prior under both backends (e.g., 31.3\% vs.\ 30.4\% relationship recall for voxel), confirming that probabilistic node association and soft evidence redistribution provide gains independent of dataset statistics. Second, upgrading from the frequentist prior to our probabilistic prior yields a further consistent improvement (e.g., 43.1\%$\to$46.0\% for PUF-Gaussian), demonstrating that the Bayesian integration of prior evidence with accumulated observations is more effective than naive frequency-based completion. The gap is especially pronounced for predicate mean recall, where the probabilistic prior's ability to modulate confidence via the existence gate $P_{ex}$ suppresses false-positive completions for rarely interacting class pairs.

% ============================================================
\section{Additional Quantitative Results}
\label{suppsec:quantitative}
% ============================================================

\subsection{Object and Predicate Performance per Class}
\label{suppsec:per_class}

Table~\ref{supptab:per_class_3dssg} reports per-class object recall on the 3DSSG test set. PUF-Gaussian achieves the highest mean object recall (65.8\%) and matches or improves FROSS on the majority of classes. Notably, classes with high visual ambiguity or low frequency, such as \texttt{chair}, \texttt{picture}, and \texttt{window}, benefit from Dirichlet evidence accumulation, which allows soft class evidence to be aggregated across multiple partial observations rather than committed to a single argmax label. PUF-Voxel exhibits a similar pattern with a mean recall of 64.1\%, confirming the representation-agnostic nature of the improvements. A few classes (e.g., \texttt{bed}, \texttt{toilet}) show minor regressions under one backend but not the other, reflecting some sensitivity of geometric overlap computation to the choice of 3D representation for objects with atypical spatial extent.

\begin{table}[htbp]
  \centering
  \caption{Per-class object recall comparison on 3DSSG test set.}
  \resizebox{\textwidth}{!}{
    \begin{tabular}{l|rrrrrrrrrrrrrrrrrrrr|r}
    \toprule
    \multicolumn{1}{c|}{\multirow{2}[2]{*}{Method}} & \multicolumn{20}{c|}{Class / Recall@1(\%)$\uparrow$}                                                                                                                        & \multicolumn{1}{c}{\multirow{2}[2]{*}{mean}} \\
          & \multicolumn{1}{l}{bath.} & \multicolumn{1}{l}{bed} & \multicolumn{1}{l}{bkshf.} & \multicolumn{1}{l}{cab.} & \multicolumn{1}{l}{chair} & \multicolumn{1}{l}{cntr.} & \multicolumn{1}{l}{curt.} & \multicolumn{1}{l}{desk} & \multicolumn{1}{l}{door} & \multicolumn{1}{l}{floor} & \multicolumn{1}{l}{ofurn.} & \multicolumn{1}{l}{pic.} & \multicolumn{1}{l}{refri.} & \multicolumn{1}{l}{show.} & \multicolumn{1}{l}{sink} & \multicolumn{1}{l}{sofa} & \multicolumn{1}{l}{table.} & \multicolumn{1}{l}{toil.} & \multicolumn{1}{l}{wall} & \multicolumn{1}{l|}{wind.} &  \\
    \midrule
    3DSSG~\cite{wald2020learning} & 25.0  & 66.7  & 0.0   & 20.0  & 51.0  & 25.8  & 50.5  & 0.0   & 47.7  & 91.4  & 14.7  & 3.4   & 22.2  & 14.3  & 25.0  & 47.4  & 42.5  & 25.9  & 51.9  & 13.1  & 31.9 \\
    SGFN~\cite{wu2021scenegraphfusion}  & 75.0  & 33.3  & 0.0   & 50.8  & 63.6  & 19.4  & 40.5  & 8.3   & 38.7  & \textbf{96.9}  & 23.0  & 11.4  & 11.1  & 0.0   & 38.3  & 55.3  & 62.3  & 51.9  & 73.0  & 13.1  & 38.3 \\
    MonoSSG~\cite{wu2023incremental} & 75.0  & \textbf{100.0} & 0.0   & 50.4  & 65.6  & 19.4  & 45.9  & 12.5  & 34.2  & \textbf{96.9}  & 25.1  & 5.7   & 0.0   & 14.3  & 38.3  & 57.9  & 59.9  & 66.7  & 76.1  & 15.5  & 43.0 \\
    SCRSSG~\cite{yeo2025statistical} & \textbf{100.0} & \textbf{100.0} & \textbf{33.3}  & \textbf{62.2}  & 63.0  & 33.3  & 54.5  & \textbf{60.0}  & 52.8  & 88.0  & 30.4  & 17.1  & \underline{66.7}  & \textbf{66.7}  & 33.3  & 60.0  & 62.3  & 80.0  & 69.3  & 37.5  & 60.5 \\
    VGfM~\cite{gay2018visual}  & 0.0   & 66.7  & 0.0   & 34.6  & 49.4  & 0.0   & 48.6  & 4.2   & 19.8  & 95.7  & 14.1  & 1.1   & 0.0   & 0.0   & 23.3  & 57.9  & 56.9  & 63.0  & \underline{78.0}  & 17.9  & 31.6 \\
    IMP~\cite{xu2017scene}   & 0.0   & 66.7  & 0.0   & 38.1  & 45.3  & 0.0   & 47.7  & 0.0   & 8.1   & 95.1  & 19.9  & 2.3   & 0.0   & 0.0   & 20.0  & 47.4  & 48.5  & 66.7  & 77.0  & 17.9  & 30.0 \\
    FROSS~\cite{hou2025fross} & \textbf{100.0} & 83.3  & \underline{28.6}  & 56.1  & 64.8  & \textbf{67.7}  & \underline{73.0}  & 29.2  & \textbf{73.3}  & 91.5  & 40.3  & \underline{41.9}  & 50.0  & 42.9  & \textbf{73.3}  & \underline{73.7}  & \underline{68.2}  & \textbf{100.0} & 60.9  & 57.5  & 63.8 \\
    \midrule
    PUF-Voxel & \textbf{100.0} & \textbf{100.0} & \underline{28.6}  & 54.3  & \underline{66.4}  & 43.3  & 64.5  & \underline{33.3}  & \underline{73.2}  & 91.0  & \underline{43.2}  & 24.7  & \textbf{100.0} & \underline{50.0}  & 61.1  & \textbf{76.7}  & 65.6  & \underline{88.5}  & 72.6  & \textbf{62.1}  & \underline{64.1} \\
    PUF-Gaussian & \textbf{100.0} & 50.0  & \underline{28.6}  & \underline{58.8}  & \textbf{72.5}  & \underline{63.3}  & \textbf{74.5}  & \underline{33.3}  & 66.1  & 94.2  & \textbf{44.8}  & \textbf{51.6}  & 80.0  & \underline{50.0}  & \underline{68.5}  & 71.2  & \textbf{69.9}  & \textbf{100.0} & \textbf{78.6}  & \underline{59.8}  & \textbf{65.8} \\
    \bottomrule
    \end{tabular}%
    }
  \label{supptab:per_class_3dssg}%
\end{table}%

\begin{table}[htbp]
  \centering
  \caption{Per-class object recall comparison on ReplicaSSG test set.}
  \resizebox{\textwidth}{!}{
    \begin{tabular}{c|lllllllllllllllll|c}
    \toprule
    Method & \multicolumn{17}{c|}{Class / Recall@1(\%)$\uparrow$}                                                                                                  & mean \\
    \midrule
    \multirow{4}[2]{*}{FROSS~\cite{hou2025fross}} & \multicolumn{1}{c}{bag} & \multicolumn{1}{c}{bskt.} & \multicolumn{1}{c}{bed} & \multicolumn{1}{c}{bench} & \multicolumn{1}{c}{bike} & \multicolumn{1}{c}{book} & \multicolumn{1}{c}{botl.} & \multicolumn{1}{c}{bowl} & \multicolumn{1}{c}{box} & \multicolumn{1}{c}{cab.} & \multicolumn{1}{c}{chair} & \multicolumn{1}{c}{clock} & \multicolumn{1}{c}{cntr.} & \multicolumn{1}{c}{cup} & \multicolumn{1}{c}{curt.} & \multicolumn{1}{c}{desk} & \multicolumn{1}{c|}{door} & \multirow{4}[2]{*}{28.8} \\
          & \textbf{25.0}  & \underline{50.0}  & 0.0   & 0.0   & 0.0   & \underline{1.5}  & \underline{9.1}  & \underline{37.5}  & \textbf{4.0}  & \textbf{14.3}  & 68.1  & \textbf{66.7}  & \textbf{40.0}  & 33.3  & \underline{9.1}  & 0.0   & 80.0  &  \\
          & \multicolumn{1}{c}{lamp} & \multicolumn{1}{c}{pil.} & \multicolumn{1}{c}{plant} & \multicolumn{1}{c}{plate} & \multicolumn{1}{c}{pot} & \multicolumn{1}{c}{rail.} & \multicolumn{1}{c}{scrn.} & \multicolumn{1}{c}{shlf.} & \multicolumn{1}{c}{shoe} & \multicolumn{1}{c}{sink} & \multicolumn{1}{c}{stand} & \multicolumn{1}{c}{table} & \multicolumn{1}{c}{toil.} & \multicolumn{1}{c}{towel} & \multicolumn{1}{c}{umb.} & \multicolumn{1}{c}{vase} & \multicolumn{1}{c|}{wind.} &  \\
          & \textbf{16.7}  & 41.5  & 47.4  & 31.2  & \underline{7.7}  & 0.0   & 0.0   & 11.1  & \underline{8.3}  & \textbf{100.0} & 0.0   & \textbf{72.2}  & \textbf{100.0} & 0.0   & \textbf{66.7}  & \textbf{38.9}  & 0.0   &  \\
    \midrule
    \multirow{4}[2]{*}{PUF-Voxel} & \multicolumn{1}{c}{bag} & \multicolumn{1}{c}{bskt.} & \multicolumn{1}{c}{bed} & \multicolumn{1}{c}{bench} & \multicolumn{1}{c}{bike} & \multicolumn{1}{c}{book} & \multicolumn{1}{c}{botl.} & \multicolumn{1}{c}{bowl} & \multicolumn{1}{c}{box} & \multicolumn{1}{c}{cab.} & \multicolumn{1}{c}{chair} & \multicolumn{1}{c}{clock} & \multicolumn{1}{c}{cntr.} & \multicolumn{1}{c}{cup} & \multicolumn{1}{c}{curt.} & \multicolumn{1}{c}{desk} & \multicolumn{1}{c|}{door} & \multirow{4}[2]{*}{\underline{30.2}} \\
          & \textbf{25.0}  & 40.0  & \textbf{25.0}  & 0.0   & \textbf{16.7}  & \underline{1.5}  & 6.1   & 34.4  & 0.0   & \textbf{14.3}  & \underline{81.2}  & \textbf{66.7}  & \textbf{40.0}  & 33.3  & \underline{9.1}  & 0.0   & \textbf{86.7}  &  \\
          & \multicolumn{1}{c}{lamp} & \multicolumn{1}{c}{pil.} & \multicolumn{1}{c}{plant} & \multicolumn{1}{c}{plate} & \multicolumn{1}{c}{pot} & \multicolumn{1}{c}{rail.} & \multicolumn{1}{c}{scrn.} & \multicolumn{1}{c}{shlf.} & \multicolumn{1}{c}{shoe} & \multicolumn{1}{c}{sink} & \multicolumn{1}{c}{stand} & \multicolumn{1}{c}{table} & \multicolumn{1}{c}{toil.} & \multicolumn{1}{c}{towel} & \multicolumn{1}{c}{umb.} & \multicolumn{1}{c}{vase} & \multicolumn{1}{c|}{wind.} &  \\
          & \textbf{16.7}  & \underline{43.4}  & \textbf{57.9}  & \underline{37.5}  & \underline{7.7}  & 0.0   & 0.0   & \underline{22.2}  & \underline{8.3}  & \underline{75.0}  & 0.0   & \textbf{72.2}  & \textbf{100.0} & 0.0   & \textbf{66.7}  & \underline{33.3}  & \underline{6.7}   &  \\
    \midrule
    \multirow{4}[2]{*}{PUF-Gaussian} & \multicolumn{1}{c}{bag} & \multicolumn{1}{c}{bskt.} & \multicolumn{1}{c}{bed} & \multicolumn{1}{c}{bench} & \multicolumn{1}{c}{bike} & \multicolumn{1}{c}{book} & \multicolumn{1}{c}{botl.} & \multicolumn{1}{c}{bowl} & \multicolumn{1}{c}{box} & \multicolumn{1}{c}{cab.} & \multicolumn{1}{c}{chair} & \multicolumn{1}{c}{clock} & \multicolumn{1}{c}{cntr.} & \multicolumn{1}{c}{cup} & \multicolumn{1}{c}{curt.} & \multicolumn{1}{c}{desk} & \multicolumn{1}{c|}{door} & \multirow{4}[2]{*}{\textbf{33.7}} \\
          & \textbf{25.0}  & \textbf{60.0}  & 0.0   & 0.0   & 0.0   & \textbf{3.7}  & \textbf{18.2}  & \textbf{40.6}  & 0.0   & \textbf{14.3}  & \textbf{85.5}  & \textbf{66.7}  & \textbf{40.0}  & \textbf{42.9}  & \textbf{12.7}  & 0.0   & \textbf{86.7}  &  \\
          & \multicolumn{1}{c}{lamp} & \multicolumn{1}{c}{pil.} & \multicolumn{1}{c}{plant} & \multicolumn{1}{c}{plate} & \multicolumn{1}{c}{pot} & \multicolumn{1}{c}{rail.} & \multicolumn{1}{c}{scrn.} & \multicolumn{1}{c}{shlf.} & \multicolumn{1}{c}{shoe} & \multicolumn{1}{c}{sink} & \multicolumn{1}{c}{stand} & \multicolumn{1}{c}{table} & \multicolumn{1}{c}{toil.} & \multicolumn{1}{c}{towel} & \multicolumn{1}{c}{umb.} & \multicolumn{1}{c}{vase} & \multicolumn{1}{c|}{wind.} &  \\
          & \textbf{16.7}  & \textbf{52.8}  & \underline{36.8}  & \textbf{43.8}  & \textbf{15.4}  & \textbf{5.9}  & \textbf{10.0}  & \textbf{33.3}  & \textbf{16.7}  & \textbf{100.0} & 0.0   & \underline{63.9}  & \textbf{100.0} & \textbf{40.0}  & \textbf{66.7}  & \underline{33.3}  & \textbf{13.3}  &  \\
    \bottomrule
    \end{tabular}%
    }
  \label{supptab:per_class_replica}%
\end{table}%

Table~\ref{supptab:per_class_replica} presents per-class object recall on the ReplicaSSG test set, which has a substantially larger and more fine-grained label space (33 classes) than 3DSSG. PUF-Gaussian achieves the highest mean recall of 33.7\%, improving over FROSS (28.8\%) across a broad range of categories. The gains are most pronounced for small or frequently occluded objects such as \texttt{bottle}, \texttt{shelf}, and \texttt{pillow}, where multiple uncertain observations must be correctly fused to recover the instance. PUF-Voxel similarly improves over FROSS (30.2\% vs.\ 28.8\%), with particularly strong gains on \texttt{chair} and \texttt{plant}. Several rare categories (\texttt{bench}, \texttt{railing}, \texttt{screen}) remain at zero or near-zero recall for all methods, indicating that these classes are limited by 2D detection failure rather than 3D fusion.

Table~\ref{supptab:per_pred_3dssg} breaks down predicate recall by class on the 3DSSG test set. Without the relationship prior, PUF-Gaussian already improves over FROSS on the two most frequent predicates, \texttt{attached to} (33.4\% vs.\ 29.4\%) and \texttt{standing on} (55.6\% vs.\ 47.2\%), demonstrating that soft evidence redistribution benefits well-observed edges. Adding the prior yields dramatic gains on \texttt{attached to} (33.4\%$\to$61.6\%) and \texttt{part of} (0.0\%$\to$13.3\%), both of which suffer from poor co-observation and are therefore heavily reliant on prior completion. The \texttt{connected to} predicate remains at 0.0\% across all online methods, likely due to 2D SGG failure. Consistent trends hold for the voxel backend, with the prior providing large lift on the same sparse predicates.

\begin{table}[htbp]
  \centering
  \caption{Per-class predicate recall comparison on 3DSSG test set.}
  \resizebox{\textwidth}{!}{
    \begin{tabular}{l|lllllll|l}
    \toprule
    \multicolumn{1}{c|}{\multirow{2}[2]{*}{Method}} & \multicolumn{7}{c|}{Predicate / Recall@1(\%)$\uparrow$}             & \multicolumn{1}{c}{\multirow{2}[2]{*}{mean}} \\
          & \multicolumn{1}{c}{attached to} & \multicolumn{1}{c}{build in} & \multicolumn{1}{c}{connected to} & \multicolumn{1}{c}{hanging on} & \multicolumn{1}{c}{part of} & \multicolumn{1}{c}{standing on} & \multicolumn{1}{c|}{supported by} &  \\
    \midrule
    3DSSG~\cite{wald2020learning} & 46.6  & 15.4  & 10.9  & 11.9  & 0.0   & 1.8   & 14.3  & 14.4 \\
    SGFN~\cite{wu2021scenegraphfusion}  & \underline{58.4}  & 33.3  & \underline{32.6}  & \textbf{26.1}  & 0.0   & 1.0   & \textbf{16.5}  & 24.0 \\
    MonoSSG~\cite{wu2023incremental} & 58.0  & 33.3  & \textbf{39.1}  & \textbf{26.1}  & \underline{12.5}  & 1.5   & \underline{15.4}  & \underline{26.6} \\
    VGfM~\cite{gay2018visual}  & 49.1  & 2.6   & 10.9  & 5.2   & 0.0   & 0.5   & 8.8   & 11.0 \\
    IMP~\cite{xu2017scene}   & 48.4  & 7.7   & 21.7  & 11.9  & 0.0   & 1.4   & 5.5   & 13.8 \\
    FROSS~\cite{hou2025fross} & 29.4  & 43.6  & 0.0   & 1.4   & 0.0   & 47.2  & 4.2   & 18.0 \\
    \midrule
    PUF-Voxel w/o prior & 32.6  & 33.8  & 0.0   & 5.8   & 6.7   & 43.7  & 4.8   & 18.3 \\
    PUF-Voxel w/ prior & 56.1  & 36.7  & 0.0   & 5.9   & 3.0   & 45.9  & 4.8   & 21.8 \\
    PUF-Gaussian w/o prior & 33.4  & \underline{51.4}  & 0.0   & 7.2   & 0.0   & \underline{55.6}  & 9.5   & 22.4 \\
    PUF-Gaussian w/ prior & \textbf{61.6}  & \textbf{56.8}  & 0.0   & 2.9   & \textbf{13.3}  & \textbf{58.1}  & \textbf{16.5}  & \textbf{28.2} \\
    \bottomrule
    \end{tabular}%
    }
  \label{supptab:per_pred_3dssg}%
\end{table}%

Table~\ref{supptab:per_pred_replica} reports per-class predicate recall on the ReplicaSSG test set, where no relationship prior is used. PUF-Gaussian outperforms FROSS on six of eight predicates and achieves the highest mean predicate recall of 26.2\%. The most substantial improvements appear on \texttt{above} (44.4\% vs.\ 22.2\%),  which depends on accurate spatial reasoning that benefits from uncertainty-aware node association. The \texttt{against} and \texttt{attached to} predicates remain at 0.0\% for all methods, reflecting the near-absence of these relations in the test scenes, and the difficulty of zero-shot transfer from Visual Genome for such spatially subtle predicates. PUF-Voxel shows modest regressions on \texttt{on} and \texttt{with} relative to FROSS, suggesting that the discrete voxel overlap metric is less effective than the Gaussian Bhattacharyya coefficient for capturing the fine spatial distinctions required by these predicates.

\begin{table}[htbp]
  \centering
  \caption{Per-class predicate recall comparison on ReplicaSSG test set.}
  \resizebox{0.8\textwidth}{!}{
    \begin{tabular}{l|llllllll|l}
    \toprule
    \multicolumn{1}{c|}{\multirow{2}[2]{*}{Method}} & \multicolumn{8}{c|}{Predicate / Recall@1(\%)$\uparrow$}                     & \multicolumn{1}{c}{\multirow{2}[2]{*}{mean}} \\
          & \multicolumn{1}{c}{above} & \multicolumn{1}{c}{against} & \multicolumn{1}{c}{attached to} & \multicolumn{1}{c}{in} & \multicolumn{1}{c}{near} & \multicolumn{1}{c}{on} & \multicolumn{1}{c}{under} & \multicolumn{1}{c|}{with} &  \\
    \midrule
    FROSS~\cite{hou2025fross} & \underline{22.2}  & 0.0   & 0.0   & \textbf{33.3}  & 28.8  & \underline{19.1}  & \textbf{10.0}  & \underline{50.0}  & \underline{20.4} \\
    PUF-Voxel & \underline{22.2}  & 0.0   & 0.0   & \textbf{33.3}  & \underline{30.1}  & 18.4  & \textbf{10.0}  & 35.0  & 17.9 \\
    PUF-Gaussian & \textbf{44.4}  & 0.0   & 0.0   & \textbf{33.3}  & \textbf{37.3}  & \textbf{24.7}  & \textbf{10.0}  & \textbf{60.0}  & \textbf{26.2} \\
    \bottomrule
    \end{tabular}%
    }
  \label{supptab:per_pred_replica}%
\end{table}%

% ============================================================
\subsection{Hyperparameters on Voxel Backend}
\label{suppsec:hyper}
% ============================================================

Table~\ref{supptab:hyper_3dssg} reports the sensitivity of the semantic bandwidth $\sigma_{se}$ and birth density $\lambda_{birth}$ on the 3DSSG validation set (w/o prior) under the voxel backend. The optimal setting is $\sigma_{se}=0.2$ and $\lambda_{birth}=0.3$, which differs slightly from the Gaussian backend ($\sigma_{se}=0.3$, $\lambda_{birth}=0.4$). This shift suggests that the voxel containment score provides a sharper spatial signal than the Gaussian Bhattacharyya coefficient, so a tighter semantic bandwidth is needed to prevent over-merging of spatially overlapping but semantically distinct objects. Similarly, a lower birth density compensates for the voxel backend's tendency to produce higher spatial overlap scores, which would otherwise suppress new-node creation. 

\begin{table}[htbp]
  \centering
  \caption{Performance comparison on 3DSSG validation set (w/o prior) with different semantic bandwidth $\sigma_{se}$ and birth density $\lambda_{birth}$. }
    \begin{tabular}{l|lllll|lllll}
    \toprule
    \multicolumn{1}{c|}{\multirow{2}[2]{*}{Recall@1(\%)$\uparrow$}} & \multicolumn{5}{c|}{$\sigma_{se}$} & \multicolumn{5}{c}{$\lambda_{birth}$} \\
          & \multicolumn{1}{c}{0.1} & \multicolumn{1}{c}{0.15} & \multicolumn{1}{c}{0.2} & \multicolumn{1}{c}{0.25} & \multicolumn{1}{c|}{0.3} & \multicolumn{1}{c}{0.2} & \multicolumn{1}{c}{0.25} & \multicolumn{1}{c}{0.3} & \multicolumn{1}{c}{0.35} & \multicolumn{1}{c}{0.4} \\
    \midrule
    Rel.  & 27.6  & 28.6  & \textbf{31.7}  & \underline{29.2}  & 28.4  & 27.9  & 30.5  & \textbf{31.7}  & \underline{31.2}  & 29.7  \\
    Obj.  & \textbf{63.1}  & 62.1  & \underline{62.6}  & 62.3  & 61.5  & 60.4  & 62.5  & 62.6  & \underline{64.0}  & \textbf{64.9}  \\
    Pred. & 31.4  & 32.6  & \textbf{35.4}  & \underline{33.1}  & 32.8  & 31.5  & 34.6  & \textbf{35.4}  & \textbf{35.4}  & 33.9  \\
    \bottomrule
    \end{tabular}%
  \label{supptab:hyper_3dssg}%
\end{table}%

\begin{table}[htbp]
  \centering
  \caption{Performance comparison on ReplicaSSG validation set with different semantic bandwidth $\sigma_{se}$ and birth density $\lambda_{birth}$. }
    \begin{tabular}{l|lllll|lllll}
    \toprule
    \multicolumn{1}{c|}{\multirow{2}[2]{*}{Recall@1(\%)$\uparrow$}} & \multicolumn{5}{c|}{$\sigma_{se}$} & \multicolumn{5}{c}{$\lambda_{birth}$} \\
          & \multicolumn{1}{c}{0.3} & \multicolumn{1}{c}{0.35} & \multicolumn{1}{c}{0.4} & \multicolumn{1}{c}{0.45} & \multicolumn{1}{c|}{0.5} & \multicolumn{1}{c}{0.3} & \multicolumn{1}{c}{0.35} & \multicolumn{1}{c}{0.4} & \multicolumn{1}{c}{0.45} & \multicolumn{1}{c}{0.5} \\
    \midrule
    Rel.  & 12.1  & \underline{12.5}  & \textbf{13.9}  & 11.8  & 10.9  & 10.5  & 11.4  & \textbf{13.9}  & 12.5  & \underline{12.9} \\
    Obj.  & \underline{23.6}  & 23.5  & \textbf{24.0}  & 22.8  & 22.4  & 20.8  & 23.0  & \underline{24.0}  & 23.6  & \textbf{24.1} \\
    Pred. & 16.8  & 18.0  & \textbf{19.3}  & \underline{18.4}  & 18.0  & 15.9  & 17.3  & \textbf{19.3}  & \underline{18.8}  & 18.4 \\
    \bottomrule
    \end{tabular}%
  \label{supptab:hyper_replica}%
\end{table}%

Table~\ref{supptab:hyper_replica} presents the corresponding hyperparameter analysis on the ReplicaSSG validation set under the voxel backend. The optimal configuration is $\sigma_{se}=0.4$ and $\lambda_{birth}=0.4$, matching the Gaussian backend on this dataset. This convergence is consistent with the zero-shot evaluation setting of ReplicaSSG, where the 2D model's class distributions are broader and less peaked than on 3DSSG, making the fusion layer less sensitive to the choice of spatial backend. The relationship recall peaks at 13.9\% and gradually degrades on both sides of the optimum, with the steepest drop occurring at high $\sigma_{se}$ values that over-merge distinct objects. The trends mirror those observed for the Gaussian backend in Table~5 of the main paper.

% ============================================================
\subsection{Sensitivity to Uncertainty Calibration}
\label{suppsec:calibration}
% ============================================================

Uncertainty mis-calibration is a known issue with the softmax outputs from deep neural networks~\cite{guo2017calibration}, \ie, the predictions can be overconfident in high-probability classes. To verify if PUF is sensitive to such mis-calibrations, we evaluate PUF with temperature scaling~\cite{guo2017calibration}, applying $T \in \{0.5, 1.0, 1.5\}$ to calibrate softmax outputs before fusion. 
Results in Tab.~\ref{tab:calibration} show robustness across all temperatures on both datasets. 

\begin{table}[htbp]
  \centering
  \caption{Performance (Recall \%) comparison with different scaling temperature T. Model is PUF-Gaussian w/o prior.}
    \begin{tabular}{l|ccc|ccc|ccc}
    \toprule
    \multicolumn{1}{c|}{\multirow{2}[1]{*}{Dataset}} & \multicolumn{3}{c|}{T=0.5} & \multicolumn{3}{c|}{T=1.0} & \multicolumn{3}{c}{T=1.5} \\
          & Rel.  & Obj.  & Pred. & Rel.  & Obj.  & Pred. & Rel.  & Obj.  & Pred. \\
    \midrule
    3DSSG & 34.6  & 70.3  & 40.3  & 33.9  & 69.0    & 39.2  & 33.6  & 69.5  & 37.8 \\
    ReplicaSSG & 27.0    & 31.4  & 36.2  & 25.3  & 31.0    & 35.6  & 25.2  & 30.5  & 36.4 \\
    \bottomrule
    \end{tabular}%
    % \vspace{-2mm}
  \label{tab:calibration}%
\end{table}%

Temperature scaling is a post-hoc uncertainty calibration method, with which PUF remains plug-and-play and training-free. More advanced methods~\cite{wang2023calibration} would give more accurate uncertainty calibration but require additional network retraining to varying degrees. The integration with these methods is therefore an interesting topic for future research.

\newpage

% ============================================================
\section{Additional Qualitative Results}
\label{suppsec:qualitative}
% ============================================================

\subsection{More Qualitative Results on 3DSSG}

Figure~\ref{fig:supp_vis_3dssg} compares PUF's qualitative outputs under the Gaussian and voxel backends on a representative 3DSSG scene. Both backends correctly recover most object instances and relationship edges present in the ground truth. The Gaussian backend produces slightly lower normalized entropy on most nodes and edges (e.g., \texttt{floor}: 0.10 vs.\ 0.11, \texttt{sofa}: 0.14 vs.\ 0.19), reflecting its finer-grained spatial overlap modeling via the Bhattacharyya coefficient. The lamp near the sofa (labeled as \texttt{other} in 3DSSG) is a rare object and misclassified with the voxel backend, but also carries high uncertainty with the Gaussian representation. Despite these minor differences, the two backends produce consistent graph topology and class predictions for this scene, corroborating the representation-agnostic nature of our framework. The entropy values annotated on each node and edge further demonstrate that PUF's Dirichlet representation provides interpretable, per-prediction confidence estimates regardless of the underlying 3D representation.

\setcounter{figure}{4}
\begin{figure}[htbp!]
    \centering
    \includegraphics[width=0.9\linewidth]{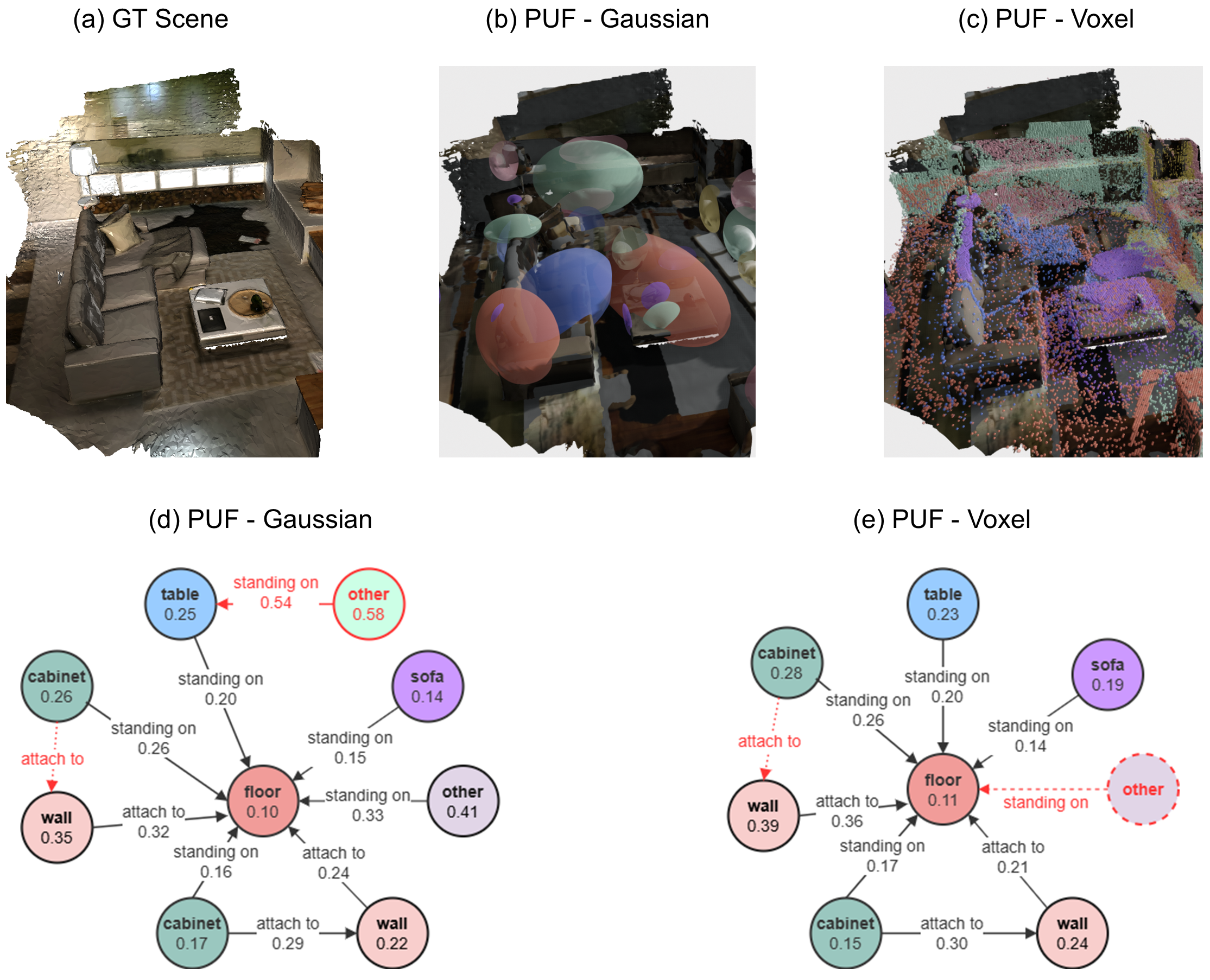}
    \caption{
        Comparison of PUF qualitative results between Gaussian and voxel 3D backend on a typical 3DSSG scene. Node colors correspond to the respective scene graphs. False and missing predictions are marked in \textcolor{red}{red}, with missing predictions in dashed line. 
    }
    \label{fig:supp_vis_3dssg}
    %\vspace{-3mm} %-3
\end{figure}

\subsection{Qualitative Results on ReplicaSSG}

Figure~\ref{fig:supp_vis_replica} visualizes PUF's predictions on two ReplicaSSG scenes (\texttt{Office0} and \texttt{Room0}) under both the Gaussian and voxel backends. In the \texttt{Office0} scene, both backends correctly detect all major objects and recover the dominant spatial relationships such as \texttt{near} and \texttt{with}. In the more cluttered \texttt{Room0} scene, both backends successfully reconstruct the majority of the scene graph structure. False predictions (marked in red) carry notably higher normalized entropy in both cases, consistent with the calibration behavior observed on 3DSSG in the main paper. A major source of error is missing objects (\eg, atypical objects \texttt{screen} in \texttt{Office0} and \texttt{table} in \texttt{Room0}) due to 2D detector failure. Although cannot be mitigated by PUF, this type of error will diminish as more powerful detection models are available. Overall, the qualitative results on ReplicaSSG confirm that PUF generalizes to unseen environments in a zero-shot setting, and that the uncertainty estimates remain informative across both backends and datasets.

\begin{figure}[htbp!]
    \centering
    \includegraphics[width=\linewidth]{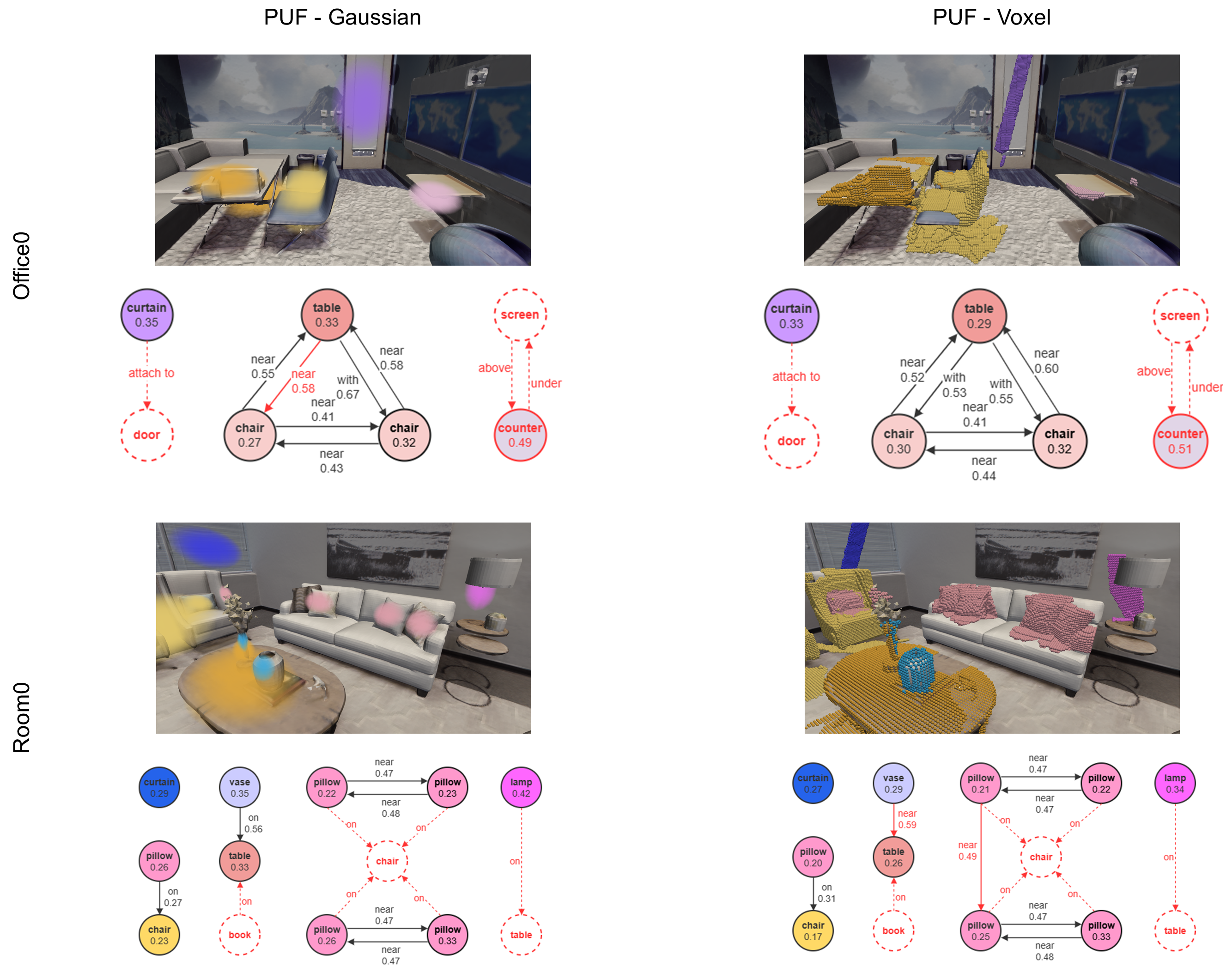}
    \caption{
        PUF qualitative results with Gaussian and voxel 3D backend on two typical ReplicaSSG scenes. Node colors correspond to the respective scene graphs. False and missing predictions are marked in \textcolor{red}{red}, with missing predictions in dashed line. 
    }
    \label{fig:supp_vis_replica}
    %\vspace{-3mm} %-3
\end{figure}

% ---- Bibliography ----
%
% BibTeX users should specify bibliography style 'splncs04'.
% References will then be sorted and formatted in the correct style.
%
\bibliographystyle{splncs04}
\bibliography{main}
\end{document}